%% file: example_paper.tex
\documentclass{article}

\usepackage{microtype}
\usepackage{graphicx}
\usepackage{booktabs} 

\usepackage{hyperref}

\usepackage[preprint]{icml2025}

\usepackage[most]{tcolorbox}

\usepackage{amsmath}
\usepackage{amssymb}
\usepackage{mathtools}
\usepackage{amsthm}

\usepackage{epigraph}
\usepackage[capitalize,noabbrev]{cleveref}

\usepackage{multirow}
\usepackage{booktabs}
\usepackage{bm}
\usepackage{xcolor}
\usepackage{caption}
\usepackage[table]{xcolor}
\usepackage{subcaption}
\usepackage{algorithm}
\usepackage{algorithmic}

\usepackage{siunitx}

\newcommand{\improve}[1]{\textcolor{green!60!black}{$\blacktriangle$#1}}
\newcommand{\stddec}[1]{\textcolor{green!60!black}{\scriptsize $\bm{\Delta\sigma=}\bm{#1}$}}

\newcommand{\kbmlift}{\mathcal{F}_{\text{KBM}}}
\newcommand{\ccpplift}{\mathcal{F}_{\text{CCPP}}}
\newcommand{\mlplift}{\mathcal{F}_{\text{MLP}}}

\usepackage{pgfplots}
\pgfplotsset{compat=1.18}
\usepackage{tikz}
\newlength{\adtframesize}
\newcommand{\adtframe}[1]{%
  \begin{minipage}[b]{\adtframesize}
    \centering
    \includegraphics[width=\adtframesize,height=\adtframesize,keepaspectratio]{#1}
  \end{minipage}%
}

\newcommand{\hypbox}[2]{%
    \begin{tcolorbox}[colback=white!98!black,colframe=white!30!black,boxsep=1.1pt,top=6.75pt]%
    \vspace{1.75pt}%
    \textbf{#1}\\[-0.575em]
    \noindent\makebox[\textwidth]{\rule{\textwidth}{0.4pt}}
    \\[0.25em]
    #2
    \end{tcolorbox}
}

\usepackage{pifont}

\theoremstyle{plain}
\newtheorem{theorem}{Theorem}[section]
\newtheorem{proposition}[theorem]{Proposition}

\theoremstyle{definition}

\theoremstyle{remark}

\usepackage[textsize=tiny]{todonotes}

\icmltitlerunning{Addressing the Waypoint--Action Gap in End-to-End Autonomous Driving via Vehicle Motion Models}

\begin{document}

\twocolumn[
\icmltitle{Addressing the Waypoint--Action Gap in End-to-End \\ Autonomous Driving via Vehicle Motion Models}



\icmlsetsymbol{equal}{*}

\begin{icmlauthorlist}
\icmlauthor{Jorge Daniel Rodr\'{i}guez-Vidal}{cvc}
\icmlauthor{Gabriel Villalonga}{cvc,uab}
\icmlauthor{Diego Porres}{cvc}
\icmlauthor{Antonio M. L\'{o}pez Pe\~{n}a}{cvc,uab}
\end{icmlauthorlist}

\icmlaffiliation{cvc}{Computer Vision Center (CVC), Barcelona, Spain}
\icmlaffiliation{uab}{Dpt. Ci\`{e}ncies de la Computaci\'{o}, Universitat Aut\'{o}noma de Barcelona (UAB), Spain}

\icmlcorrespondingauthor{Jorge Daniel Rodr\'{i}guez-Vidal}{jdrodriguez@cvc.uab.cat}

\icmlkeywords{Machine Learning, ICML}

\vskip 0.3in
]



\printAffiliationsAndNotice{\icmlEqualContribution} 

\input{sec/0_abstract}
\input{sec/1_intro}

\input{sec/2_related_work}
\input{sec/3a_new_method}
\input{sec/4_experiments}
\input{sec/5_conclusions}
\bibliography{main}
\bibliographystyle{icml2025}
\input{sec/6_appendix}

\end{document}

%% file: sec/0_abstract.tex
\begin{abstract}
End-to-End Autonomous Driving (E2E-AD) systems are typically grouped by the nature of their outputs: (i) waypoint-based models that predict a future trajectory, and (ii) action-based models that directly output throttle, steer and brake. Most recent benchmark protocols and training pipelines are waypoint-based, which makes action-based policies harder to train and compare, slowing their progress. To bridge this waypoint–action gap, we propose a novel, differentiable vehicle-model framework that rolls out predicted action sequences to their corresponding ego-frame waypoint trajectories while supervising in waypoint space. Our approach enables action-based architectures to be trained and evaluated, for the first time, within waypoint-based benchmarks without modifying the underlying evaluation protocol. We extensively evaluate our framework across multiple challenging benchmarks and observe consistent improvements over the baselines. In particular, on NAVSIM \emph{navhard} our approach achieves state-of-the-art performance. Our code will be made publicly available upon acceptance.
\end{abstract}


%% file: sec/1_intro.tex
\begin{figure}[t]
    \centering
    \includegraphics[width=\columnwidth]{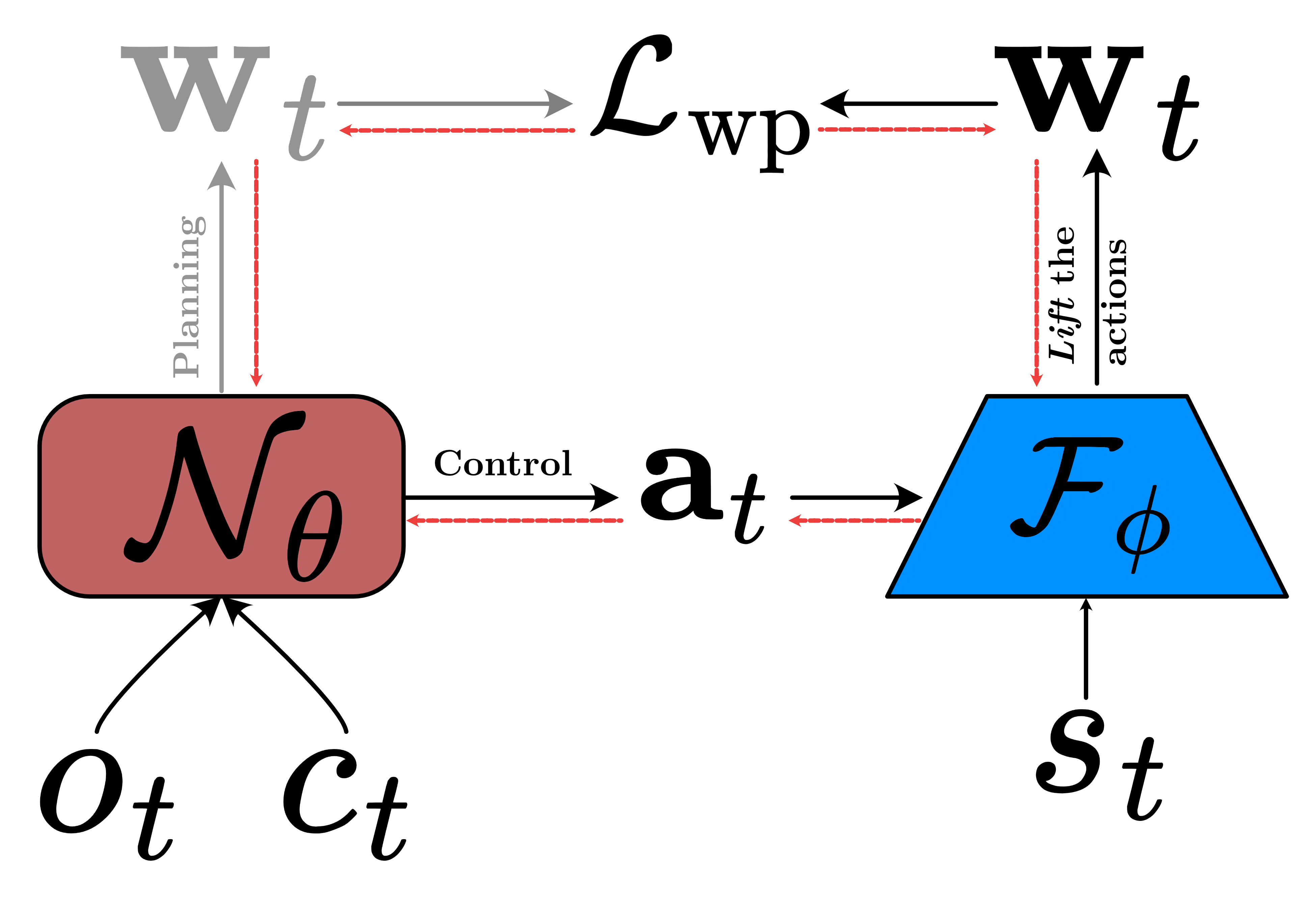}
    \caption{
\textbf{Bridging the \emph{Waypoints--Actions} gap with our differentiable framework.}
Black arrows denote the forward pass and red arrows denote gradient flow.
Here, $\mathcal{N}_\theta$ is the policy network, which takes as input the current observations $o_{t}$ and high-level command $c_{t}$. Two paradigms exist for the outputs of the policy: \emph{planning-based}, which outputs the future waypoints $\mathbf{w}_{t}$ the ego vehicle will follow, and \emph{control-based}, which outputs low-level controls $\mathbf{a}_{t}$. $\mathcal{F}_\phi$, the lifting operator, will take these actions along the state $s_{t}$ of the vehicle, and lift them to a set of waypoints. This allows us to use the gradients from a waypoint or planning-based loss on an action-based policy.
}

    \label{fig:short}
\end{figure}

\section{Introduction}
\label{sec:intro}

\epigraph{Give me the positions and velocities of all the particles in the universe, and I will predict the future.}{\textit{Marquis Pierre Simon de Laplace}} 

End-to-end autonomous driving (E2E-AD) has emerged as a promising paradigm that directly maps raw sensor observations to driving commands. Within this scope, two paradigms are commonly used: \emph{action-based} methods that predict low-level controls such as steering and longitudinal acceleration, and \emph{waypoint-based} methods that predict a future trajectory. In this work, we focus on the action-based paradigm, as it alleviates the need for high-precision GPS and supports deployment in rural areas where connectivity may be slow or unreliable.

Historically, action-based and waypoint-based approaches evolved in parallel, but in recent years the E2E-AD research community has increasingly focused on waypoint-based models, creating a \emph{Waypoint--Action gap}. This imbalance is further amplified by waypoint-only benchmarks ~\cite{Dauner24NAVSIM,Cao25PseudoSimulation,Runsheng25Waymo}, in which agents are required to output fixed-horizon waypoint trajectories, while direct control outputs are not supported. Consequently, they only accommodate waypoint-based architectures, while action-based policies cannot be evaluated without introducing an additional \emph{non-standard action interface}, e.g. by discretising the control space via a tokenised library ~\cite{Wu24SMART,Li24HydraMDP}. This motivates the need for a general mechanism that can adapt action-based policies to waypoint-based benchmarks.

In this work, we introduce an \emph{analytic bridge} between action and waypoint-based policy representations by explicitly \emph{lifting} low-level control outputs (i.e., throttle, steering, brake) into waypoint trajectories through a differentiable vehicle-model framework, enabling fair training, evaluation, and comparison under existing waypoint-based benchmarks without modifying their interfaces or protocols. Concretely, we instantiate this bridge with two models: the Kinematic Bicycle Model (KBM) ~\cite{Rajamani12VehicleControl,Kong15VehicleModels,Polack17KinematicBicycle} and a continuous-curvature kinematic model inspired by the Continuous Curvature Path Planner (CCPP)~\cite{Scheuer96ContinuousCurvature}. The KBM is a single-track vehicle model that captures the nonholonomic constraints of car-like robots~\cite{Dubins57,ReedsShepp90,LaValle06Planning} and underpins many path-tracking and planning algorithms used in practice~\cite{Coulter92PurePursuit,Thrun06Stanley,Snider09Steering}. The CCPP is a collision-free planner that generates smooth trajectories by composing clothoid arcs. Since clothoid curvature varies linearly with arc length, CCPP avoids the curvature discontinuities of straight--circular-arc paths that would otherwise force abrupt steering reorientation \cite{Scheuer96ContinuousCurvature}.

We evaluate our framework on four challenging benchmarks: (i) NAVSIM \texttt{navtest}~\cite{Dauner24NAVSIM}, (ii) NAVSIM \texttt{navhard}~\cite{Cao25PseudoSimulation}, (iii) Bench2Drive~\cite{Jia24Bench2Drive}, and (iv) a CARLA-based evaluation protocol~\cite{Dosovitskiy17,Porres24GuidingAttention}. Across NAVSIM \texttt{navhard}, NAVSIM \texttt{navtest}, Bench2Drive and CARLA, action-based policies equipped with our framework match or outperform waypoint-based baselines: our framework achieves vision-only state-of-the-art on \texttt{navhard}, come within $1.5\%$ of the best vision-only model on \texttt{navtest}, improve Bench2Drive baseline model DS by up to $61.1\%$, and obtain the strongest loss--outcome correlation under the CARLA-based evaluation protocol.

We summarize our main contributions as follows:

\begin{itemize}
    \item \textbf{Waypoint-based training objective.} We show that coupling action policies to our framework and supervising them with a waypoint-space loss leads to \emph{more stable evaluation}, \emph{higher closed-loop performance}, and a \emph{stronger correlation} with driving outcomes than standard action-space objectives.
    \item \textbf{Differentiable and deterministic framework.} To the best of our knowledge, it is the first time that a \emph{deterministic} and \emph{differentiable} vehicle-dynamics framework is introduced to lift low-level actions to waypoints.
    \item \textbf{Unified comparison across paradigms.} We instantiate our framework with two vehicle models: a Kinematic Bicycle Model (KBM) and a Continuous Curvature Path Planner (CCPP), which are able to lift low-level actions into waypoint trajectories, enabling \emph{action-based} and \emph{waypoint-based} policies to be trained and evaluated under the same waypoint protocol without modifying existing benchmarks.
\end{itemize}

%% file: sec/2_related_work.tex
\section{Related Work}
\label{sec:related_work}

\textbf{Origins of car-like kinematics.}
Foundational studies of curvature-bounded motion for car-like systems predate autonomous driving. Dubins characterised shortest forward-only paths under bounded curvature \yrcite{Dubins57}, and Reeds extended the result to cars that can drive in reverse \yrcite{ReedsShepp90}. These works established the nonholonomic nature of car-like motion and motivated simplified geometric models used to this day. Comprehensive treatments in robotics and vehicle dynamics texts later formalised the kinematics and constraints of car-like robots, including front-wheel steering and Ackermann geometry, paving a direct path to the single-track (``bicycle'') abstraction used for planning and control in practice~\cite{LaValle06Planning,Rajamani12VehicleControl}. More recently, Hanselmann \emph{et al.} \yrcite{Hanselmann22KING} used a kinematic bicycle model as a differentiable proxy for the non-differentiable CARLA simulator. While conceptually related, our work is different: we integrate differentiable vehicle models directly into the training pipeline to bridge action and waypoint-based representations for benchmarking, whereas prior work employs such proxies primarily for gradient-based scenario generation.

\noindent\textbf{Learning-based E2E driving: actions vs.\ waypoints.}
Early end-to-end (E2E) driving work demonstrated that a single network can map raw sensor inputs and a high-level navigation command directly to continuous low-level controls (throttle, steer, brake), as in \emph{Conditional Imitation Learning} (CIL) and its CARLA behavior-cloning extension \emph{CILRS} \cite{Dosovitskiy17,Codevilla18CIL,Codevilla19BC}. More recent action-based baselines such as \emph{CIL++} strengthened this paradigm with higher-resolution multi-view perception and attention \cite{Xiao2023cilpp}. Complementarily, \emph{MILE} advanced action-based E2E driving by introducing \emph{model-based imitation learning}: it jointly learned a compact latent \emph{world model} and an ego policy, using 3D geometry as an inductive bias and enabled imagination-based rollouts that could be decoded into bird's-eye-view semantic representations \cite{Hu22MILE}. A number of E2E systems leverage LiDAR-only or LiDAR+RGB sensing, but to the best of our knowledge, there are no open source action-based baselines that output continuous control e.g. \cite{Mirzaie25InterpretableE2E}; other LiDAR-based action policies focus on steering-only prediction \cite{Wang21FlowDriveNet}. In parallel, a dominant line of modern E2E-AD predicts \emph{waypoint trajectories} and relies on a downstream controller to realize actions: \emph{ChauffeurNet} popularized waypoint prediction with controller tracking, \emph{Learning by Cheating} distilled privileged planning into vision-only waypoint policies, and \emph{TransFuser} further improved robustness by fusing multi-modal inputs with transformer-based architectures for trajectory/waypoint prediction \cite{Bansal19ChauffeurNet,Chen20LearningByCheating,Chitta22TransFuser}. Motivated by the need to represent multi-modal futures and uncertainty, diffusion models e.g., \emph{DDPM} and \emph{DDIM} \cite{Ho20DDPM,Song21DDIM} are increasingly adopted as waypoint/trajectory generators; inspired by robotic policies e.g., \emph{Diffusion Policy} \cite{Chi23DiffusionPolicy}, recent autonomous-driving methods leveraged diffusion-based decoders for trajectory planning  \cite{Liao24,Zhao25DiffE2E}. This growing skew toward waypoint outputs \cite{Chen24E2ESurvey} motivates mechanisms that bridge action-space policies and waypoint-based training/evaluation without redesigning benchmark protocols..

%% file: sec/3a_new_method.tex
\section{Expressing dynamical vehicle models in a common framework}
\label{sec:newmethod}

\hypbox{Problem setup.}{
Consider an ego vehicle $\mathcal{V}$ moving along a trajectory in the plane $\mathcal{P}\subset \mathbb{R}^{2}$. At each decision step, the vehicle executes a sequence of low-level actions over a control horizon spanning $C_{f}$ timesteps. \textbf{Can we accurately predict the vehicle's state after executing each action in this sequence?}

We model the driving process as a discrete-time system indexed by decision time $t \in \mathbb{N}$.
At each decision time $t$, an action-based end-to-end autonomous driving (E2E-AD) policy network $\mathcal{N}_\theta$ takes as input the current observation $o_t$ and a high-level command $c_t$ (e.g., turn left), and outputs a sequence of actions over a control horizon of length $C_f \in \mathbb{N}$:

\begin{equation}
\mathbf{a}_t
=
\left(a_{t,0}, a_{t,1}, \dots, a_{t,C_f-1}\right)
\in
\mathbb{R}^{C_{f} \times d_{u}},
\label{eq:raw_action_seq}
\end{equation}

where $d_u$ is the action dimension and $\Delta t > 0$ is the fixed control interval between consecutive actions. Each action $a_{t,k}$ corresponds to a ground-truth waypoint $w_{t,k+1}^{\text{gt}}\in\mathbb{R}^{2}$ representing the vehicle's actual position after executing that action. We collect these waypoints into the ground-truth trajectory:

\begin{equation}
\mathbf{w}_{t}^{\text{gt}}
=
\left(w_{t,1}^{\text{gt}}, w_{t,2}^{\text{gt}}, \dots, w_{t,C_f}^{\text{gt}}\right)
\in
\mathbb{R}^{C_{f} \times 2}.
\label{eq:gt_waypoint_seq}
\end{equation}
}

Dynamical vehicle models such as the Kinematic Bicycle Model (KBM) approximate vehicle motion using simplified representations that capture essential dynamics without requiring exhaustive physical parameters. However, different models use different state representations, control inputs, and integration methods, making it challenging to compare or combine them directly. To address this, we propose a unified framework that encapsulates various dynamical vehicle models as \emph{lifting operators}.

\subsection{Action-to-waypoint Lifting Operator}\label{subsec:lifting_operator}

In order to convert the action sequence $\mathbf{a}_t$ into a waypoint trajectory over the same horizon, we define the \emph{lifting operator}
\begin{equation}
\mathcal{F}_{\phi} :
\mathbb{R}^{C_{f} \times d_{u}} \times \mathcal{S} \longrightarrow
\mathbb{R}^{C_{f} \times 2},
\qquad
(\mathbf{a}_t, s_t) \mapsto \mathbf{w}_t,
\label{eq:general_lifting}
\end{equation}
where $\phi$ collects vehicle/model parameters and $s_t \in \mathcal{S}$ contains the minimal initial state required for propagation (e.g., initial speed). The output $\mathbf{w}_t = (w_{t,1}, \dots, w_{t,C_f})$ is a sequence of ego-frame waypoints, with $w_{t,k} \in \mathbb{R}^2$ representing the planar position at the end of the $k$-th interval.

Any lifting operator $\mathcal{F}_\phi$ can be decomposed into three modular, independently specifiable components:

\paragraph{Control Activation $\psi_\phi: \mathbb{R}^{d_u} \to \mathcal{U}$.} 
Maps raw network outputs $a_{t,k} \in \mathbb{R}^{d_u}$ to bounded vehicle controls $u_{t,k} \in \mathcal{U}$ via smooth activation functions (e.g., sigmoid). This ensures that predicted controls lie within the vehicle's operational envelope and enables stable gradient-based training.

\paragraph{Dynamics rollout $f_\phi: \mathcal{Z} \times \mathcal{U} \to \mathcal{Z}$.} 
Propagates the kinematic state $z_{t,k} \in \mathcal{Z}$ forward by one discrete time step $\Delta t$ using vehicle dynamics. For analytical models (KBM, CCPP), this is achieved by discretising continuous-time ODEs using a numerical solver (e.g., Euler or RK4). For learned models (MLP), the dynamics rollout is implicit in the network weights.

\paragraph{Pose projection $h: \mathcal{Z} \to \mathbb{R}^2$.} 
Extracts the ego-frame waypoint $(x, y) \in \mathbb{R}^2$ from the full state $z_{t,k}$, which may include auxiliary variables such as speed or curvature. 
Note $z_{t,0}$ is initialized from $s_t$ to satisfy ego-frame anchoring at $k=0$, i.e., $h(z_{t,0}) = (0, 0)$.

The complete forward pass from actions to waypoints is:

\begin{align}
\mathcal{F}_\phi(\mathbf{a}_t, s_t) &= \left(h(z_{t,1}), \ldots,h(z_{t,C_f})\right)
\label{eq:lifting_explicit}\\
 z_{t,k+1} &= f_\phi\left(z_{t,k}, \psi_\phi(a_{t,k}), \Delta t\right)
 \label{eq:next_rollout}
\end{align}

\paragraph{Modularity and extensibility.} 
The decomposition into $(\psi_\phi, f_\phi, h)$ offers two key advantages. First, individual components can be modified independently: changing the numerical integration scheme (e.g., Euler $\to$ RK4) requires no retraining, and replacing the dynamics model (e.g., KBM $\to$ CCPP) does not alter the training objective. Second, new motion models can be integrated by specifying their three components, provided they satisfy mild regularity conditions ensuring determinism and differentiability. This modular design enables systematic exploration of the design space and facilitates future improvements.

\begin{proposition}[Determinism and differentiability of lifting operators]
\label{prop:lifting-smooth}
Let $\mathcal{F}_\phi: \mathbb{R}^{C_f \times d_u} \times \mathcal{S} \to \mathbb{R}^{C_f \times 2}$ be a lifting operator decomposed as $\psi_\phi: \mathbb{R}^{d_u} \to \mathcal{U}$, $f_\phi: \mathcal{Z} \times \mathcal{U} \to \mathcal{Z}$, and $h: \mathcal{Z} \to \mathbb{R}^2$, where $\mathcal{U}$ is the control space, $\mathcal{Z}$ is the state space, and $\mathcal{S}$ is the initial state space.

Suppose: (a) $\psi_\phi$ is $C^1$ with bounded range (ensuring physical control limits), (b) $f_\phi$ is $C^1$ in both arguments, and (c) $h$ is $C^1$. Then for fixed initial state $s_t \in \mathcal{S}$, the lifting operator $\mathcal{F}_\phi(\cdot, s_t)$ is: (i) \emph{deterministic}: each action sequence $\mathbf{a}_t \in \mathbb{R}^{C_f \times d_u}$ produces a unique waypoint sequence $\mathbf{w}_t \in \mathbb{R}^{C_f \times 2}$, (ii) \emph{continuously differentiable}: $\mathcal{F}_\phi(\cdot, s_t) \in C^1(\mathbb{R}^{C_f \times d_u}, \mathbb{R}^{C_f \times 2})$.
\end{proposition}

\begin{proof}
The result follows by composition of $C^1$ functions and induction over the control horizon. See Appendix~\ref{app:proof} for details.
\end{proof}

\begin{table*}[ht]
\caption{\textbf{Design space of action-to-waypoint lifting operators.} All instantiations implement $\mathcal{F}_\phi(\mathbf{a}_t, s_t)$ via three modular components: control activation $\psi_\phi$, dynamics rollout $f_\phi$, and pose projection $h$. This decomposition enables mix-and-match of design choices. Vehicle parameters (e.g., $L$, $\delta_{\max}$) are physical specifications, not hyperparameters. ``—'' denotes not applicable or implicit.}
\label{tab:design-space}
\vskip 0.1in
\centering
\resizebox{0.57\textwidth}{!}{
\begin{tabular}{@{}l@{\hspace{8pt}}r@{\hspace{12pt}}r@{\hspace{12pt}}r@{}}
\toprule
& $\mathbf{\kbmlift}$ & $\mathbf{\ccpplift}$ & $\mathbf{\mlplift}$ \\
\midrule
\multicolumn{4}{@{}l}{\textit{Control Activation} $\psi_\phi: \mathbb{R}^{d_u} \to \mathcal{U}$} \\[2pt]
\quad Action channels $a_{t,k}$ & $(\tau, \sigma, \beta)$ & $(\tau, \xi, \beta)$ & $(\tau, \sigma, \beta)$ \\
\quad Longitudinal activation & sigmoid & sigmoid & \textbf{learned} \\
\quad Lateral activation & tanh & tanh & \textbf{learned} \\
\quad Physical controls $u_{t,k}$ & $(a^{\text{lon}}, \delta)$ & $(a^{\text{lon}}, \varsigma)$ & — \\
\quad Lateral control & $\delta = \delta_{\max}\hat{\sigma}$ & $\varsigma = \varsigma_M \hat{\xi}$ & — \\
\midrule
\multicolumn{4}{@{}l}{\textit{State Space} $\mathcal{Z}$} \\[2pt]
\quad State variables $z_{t,k}$ & $(x, y, \theta, v)$ & $(x, y, \theta, \kappa)$ & implicit \\
\quad Auxiliary variable & speed $v$ [$\unit{m.s^{-1}}$] & curvature $\kappa$ [$\unit{m^{-1}}$] & — \\
\quad Initial state $s_t$ & $v_0$ & $(v_0, \kappa_0)$ & — \\
\midrule
\multicolumn{4}{@{}l}{\textit{Dynamics Rollout} $f_\phi: \mathcal{Z} \times \mathcal{U} \to \mathcal{Z}$} \\[2pt]
\quad Continuous dynamics & bicycle (Eq.~\ref{eq:kbm-ct-x}-\ref{eq:kbm-ct-v}) & clothoid (Eq.~\ref{eq:dxds_ccpp}-\ref{eq:dkds_ccpp}) & — \\
\quad Integration variable & time $t$ & arc-length $s$ & — \\
\quad Discretisation & Euler / RK4 & Euler / RK4 & — \\
\quad Substeps per $\Delta t$ & 1 & $n_{\text{int}}$ & 1 \\
\quad Saturations & none & $v \geq 0$, $|\kappa| \leq \kappa_M$ & — \\
\midrule
\multicolumn{4}{@{}l}{\textit{Pose Projection} $h: \mathcal{Z} \to \mathbb{R}^2$} \\[2pt]
\quad Readout & $h(x,y,\theta,v) = (x,y)$ & $h(x,y,\theta,\kappa) = (x,y)$ & identity \\
\bottomrule
\end{tabular}
}
\vskip -0.1in
\end{table*}

\subsection{Lifting Operator Instantiations}\label{subsec:lifting}
Table~\ref{tab:design-space} summarizes three instantiations of the lifting framework. We now detail their key distinguishing features.

\paragraph{KBM Lifting $\kbmlift$}
The Kinematic Bicycle Model approximates the vehicle as a single-track bicycle with state $z_{t,k} = (x, y, \theta, v)$ comprising position, heading, and speed. Action channels $(\tau, \sigma, \beta)$ represent throttle, steering, and brake inputs, transformed via sigmoid/tanh activations into longitudinal acceleration $a^{\text{lon}}$ and steering angle $\delta = \delta_{\max}\hat{\sigma}$. Time-based integration (Euler or RK4) propagates the dynamics with single substep per $\Delta t$. See Appendix~\ref{app:kbm_method} for the continuous dynamics (Eq.~\ref{eq:kbm-ct-x}-\ref{eq:kbm-ct-v}).

\paragraph{CCPP Lifting $\ccpplift$}
The Continuous Curvature Path Planner uses clothoid curves to ensure curvature continuity. Unlike KBM's time parameterization, CCPP integrates along arc-length $s$ with state $z_{t,k} = (x, y, \theta, \kappa)$ including curvature. Action channels $(\tau, \xi, \beta)$ produce sharpness $\varsigma = \varsigma_M\hat{\xi}$ (curvature rate). Multiple substeps per $\Delta t$ enhance smoothness, with saturations $v \geq 0$ and $|\kappa| \leq \kappa_M$ ensuring physical feasibility. See Appendix~\ref{app:ccpp_method} for the arc-length dynamics (Eq.~\ref{eq:dxds_ccpp}-\ref{eq:dkds_ccpp}).

\paragraph{MLP Lifting $\mlplift$}
The MLP operator directly learns the action-to-waypoint mapping via a multi-layer perceptron, bypassing explicit dynamics. While sacrificing physical guarantees, this data-driven approach may capture behaviors beyond simplified kinematic models. See Appendix~\ref{app:extraimpdet} for details.

\subsection{End-to-End Training with Waypoint Supervision}
\label{subsec:training}

We train the policy network $\mathcal{N}_\theta$ end-to-end using waypoint supervision. At each decision time $t$, the network predicts actions $\mathbf{a}_t$ from observations $o_t$ and commands $c_t$. These actions are passed through the differentiable lifting operator $\mathcal{F}_\phi$ to produce predicted waypoints $\mathbf{w}_t$, which are supervised against ground-truth waypoints $\mathbf{w}_t^{\text{gt}}$ from the dataset (Eq.~\ref{eq:gt_waypoint_seq}). Gradients flow back through the lifting operator to update the policy parameters $\theta$, while the vehicle model parameters $\phi$ remain fixed.

\paragraph{Waypoint loss.}
We supervise the predicted trajectory using a weighted $L_1$ loss over waypoint positions:
\begin{equation}
    \mathcal{L}_{\mathrm{wp}}
    \;=\;
    \frac{1}{\sum_{k=1}^{C_f}\alpha_k}
    \sum_{k=1}^{C_f}
    \alpha_k\,\bigl\|\hat{w}_{t,k} - w^{\mathrm{gt}}_{t,k}\bigr\|_{1},
\label{eq:training_obj}
\end{equation}
where $\hat{w}_{t,k}, w^{\mathrm{gt}}_{t,k} \in \mathbb{R}^2$ are the $k-$th predicted and ground-truth planar positions at time $t$, respectively, and $\alpha_k \geq 0$ are optional temporal weights (e.g., to emphasize near-term predictions). This position-only objective ignores heading errors, which are less critical for closed-loop control. We use $L_1$ for robustness to outliers. Although intermediate states include heading $\theta_{t,k}$, we exclude it from supervision to enable fair comparison with existing methods.
 We leave this direction for future work.

\paragraph{Training pipeline.}
Algorithm~\ref{alg:lifting} shows the complete forward pass. The policy outputs actions, the lifting operator unrolls them into a trajectory via the three-component decomposition $(\psi_\phi, f_\phi, h)$, and the loss compares predicted positions to ground truth. All components are differentiable, enabling standard backpropagation through the entire pipeline to optimize $\theta$.

\begin{algorithm}[t]
\small
\caption{Training Forward Pass}
\label{alg:lifting}
\begin{algorithmic}[1]
\REQUIRE Observation $o_t$, command $c_t$, initial state $s_t$, ground-truth $\mathbf{w}_t^{\text{gt}}$, \\
\hspace{1.6em} policy $\mathcal{N}_\theta$, lifting operator $\mathcal{F}_\phi = (\psi_\phi, f_\phi, h)$
\ENSURE Waypoint loss $\mathcal{L}_{\mathrm{wp}}$
\STATE $\mathbf{a}_t \gets \mathcal{N}_\theta(o_t, c_t)$ \hfill \textcolor{gray}{// Policy network}
\STATE $z_0 \gets \textsc{Init}(s_t)$ \hfill \textcolor{gray}{// Ego-frame state}
\FOR{$k = 0, \ldots, C_f - 1$}
    \STATE $u_k \gets \psi_\phi(a_{t,k})$ \hfill \textcolor{gray}{// Control activation}
    \STATE $z_{k+1} \gets f_\phi(z_k, u_k, \Delta t)$ \hfill \textcolor{gray}{// Dynamics rollout}
    \STATE $\hat{w}_{t,k+1} \gets h(z_{k+1})$ \hfill \textcolor{gray}{// Pose projection}
\ENDFOR
\STATE $\mathcal{L}_{\mathrm{wp}} \gets \textsc{WaypointLoss}(\hat{\mathbf{w}}_t, \mathbf{w}_t^{\text{gt}})$ \hfill \textcolor{gray}{// Eq.~\ref{eq:training_obj}}
\STATE \textbf{return} $\mathcal{L}_{\mathrm{wp}}$
\end{algorithmic}
\end{algorithm}

%% file: sec/4_experiments.tex
\section{Experiments}
\label{sec:experiments}

\subsection{Benchmarks}
\label{subsec:navsim_benchmark}

We primarily evaluate our proposed framework on \textbf{NAVSIM v1} \texttt{navtest} \cite{Dauner24NAVSIM} and \textbf{NAVSIM v2} \texttt{navhard} \cite{Cao25PseudoSimulation}. NAVSIM is a recent data-driven autonomous driving simulation and benchmarking framework built from large-scale real-world logs.

We further validate our approach in CARLA~\cite{Dosovitskiy17}, selecting two additional benchmarks: (i) \textbf{Bench2Drive}, where we report \emph{Dev10} results ~\cite{Jia24Bench2Drive,Jia25DriveTransformer}, and (ii) a \textbf{CARLA-based evaluation protocol}~\cite{Zhang21RLICoach,Porres24GuidingAttention}. Additional benchmark details are provided in the Appendix~\hyperref[app:extrabeninfo]{\ref*{app:extrabeninfo}}.

\subsection{Implementation Details}
\label{subsec:impl_details}

\paragraph{Training setup.} For $\kbmlift$, we use a fixed sampling interval $\Delta t=\qty{0.5}{\second}$, a maximum steering angle $\delta_{\max}=\qty{0.6}{\radian}$, longitudinal gain $a^{\mathrm{lon}}=1.0$ and a wheelbase $L=\qty{2.9}{\meter}$, which are kept constant across all experiments.

For $\ccpplift$, we also use a fixed sampling interval $\Delta t=\qty{0.5}{\second}$, initial curvature $\kappa_0=\qty{0}{\per\meter}$, curvature bound $\kappa_M=\qty{0.4}{\per\meter}$, sharpness bound $\varsigma_M=\qty{0.1}{\per\meter\squared}$, longitudinal gain $a^{\mathrm{lon}}=1.0$, and $n_{\mathrm{int}}=5$ uniform arc-length substeps per control interval.

Regarding the training objective (\ref{eq:training_obj}), we use $\alpha_k=1$ in all experiments. Further details are provided in Appendices~\ref{app:kbm_method},~\ref{app:ccpp_method}, and~\ref{app:extraimpdet}.

\begin{table*}[t]
\small
\centering
\caption{\textbf{Training-dynamics robustness (Euler vs RK4): Top-$k$ Driving Score (DS).}
For each model, we select the top-$k$ checkpoints and report the mean $\pm$ standard deviation of DS for $k\in\{3,5,10\}$.
For each entry, we additionally report the absolute standard-deviation reduction from Euler to RK4 at the same $k$
(i.e., $\Delta \sigma := \sigma_{\text{Euler}}-\sigma_{\text{RK4}}$).
}
\label{tab:topk_ds_euler_vs_rk4_euler2rk4_stddec}

\setlength{\tabcolsep}{6.5pt}
\renewcommand{\arraystretch}{1.05}

\begin{tabular}{lcccccc}
\toprule
\multirow{2}{*}{\textbf{Model}} &
\multicolumn{3}{c}{\textbf{Euler DS} $\uparrow$} &
\multicolumn{3}{c}{\textbf{RK4 DS} $\uparrow$} \\
\cmidrule(lr){2-4}\cmidrule(lr){5-7}
& \textbf{Top-3} & \textbf{Top-5} & \textbf{Top-10}
& \textbf{Top-3} & \textbf{Top-5} & \textbf{Top-10} \\
\midrule
MILE
& $48.5^{\pm 3.3}$ & $46.5^{\pm 3.7}$ & $42.8^{\pm 4.7}$
& $48.5^{\pm 3.3}$ & $46.5^{\pm 3.7}$ & $42.8^{\pm 4.7}$ \\

\quad w/ $\kbmlift$
& $57.3^{\pm 2.0}$ & $55.7^{\pm 2.7}$ & $52.9^{\pm 3.6}$
& $67.9^{\pm 1.4}$ (\stddec{0.6}) & $66.7^{\pm 2.0}$ (\stddec{0.7}) & $64.3^{\pm 3.1}$ (\stddec{0.5}) \\

\quad w/ $\ccpplift$
& $59.2^{\pm 4.5}$ & $57.0^{\pm 4.4}$ & $53.1^{\pm 5.0}$
& $57.5^{\pm 2.1}$ (\stddec{2.4}) & $56.5^{\pm 2.0}$ (\stddec{2.4}) & $54.2^{\pm 3.6}$ (\stddec{1.4}) \\

\midrule
CIL++
& $32.6^{\pm 8.9}$ & $26.9^{\pm 11.0}$ & $13.7^{\pm 15.7}$
& $32.6^{\pm 8.9}$ & $26.9^{\pm 11.0}$ & $13.7^{\pm 15.7}$ \\

\quad w/ $\kbmlift$
& $63.5^{\pm 2.6}$ & $62.5^{\pm 2.3}$ & $60.1^{\pm 3.5}$
& $62.8^{\pm 0.8}$ (\stddec{1.8}) & $62.3^{\pm 0.9}$ (\stddec{1.4}) & $59.6^{\pm 3.3}$ (\stddec{0.2}) \\

\quad w/ $\ccpplift$
& $68.1^{\pm 6.3}$ & $65.8^{\pm 5.5}$ & $61.1^{\pm 6.8}$
& $67.2^{\pm 0.6}$ (\stddec{5.7}) & $66.2^{\pm 1.5}$ (\stddec{4.0}) & $63.4^{\pm 4.3}$ (\stddec{2.5}) \\

\bottomrule
\end{tabular}
\end{table*}

\begin{table}[t]
\small
\caption{\textbf{Closed-loop performance of our RK4-based framework in Bench2Drive.} We report Bench2Drive \emph{Dev10} results with and without our RK4-based framework, averaging closed-loop metrics over three evaluation seeds. For the best checkpoint, we report mean $\pm$ standard deviation for Driving Score (\textbf{DS} $\uparrow$) and Route Completion (\textbf{RC} $\uparrow$).}
\centering
\label{tab:kbm_dev10_rk4}
\begin{tabular}{lll}
\toprule
\textbf{Model} & \textbf{DS} $\uparrow$ & \textbf{RC} $\uparrow$ \\
\midrule
MILE                 & $51.7^{\pm 1.2}$ & $56.6^{\pm 1.2}$ \\
\quad w/ $\kbmlift$           & $67.7^{\pm 1.2}$ (\improve{16.0}) & $84.9^{\pm 0.0}$(\improve{28.3})\\
\quad w/ $\ccpplift$          & $58.6^{\pm 1.2}$ (\improve{6.9}) & $80.3^{\pm 0.0}$(\improve{23.7})\\
\midrule
CIL++                & $43.6^{\pm 1.5}$ & $52.0^{\pm 1.0}$ \\
\quad w/ $\kbmlift$          & $62.8^{\pm 0.8}$(\improve{19.2}) & $83.9^{\pm 0.1}$(\improve{31.9}) \\
\quad w/ $\ccpplift$          & $66.9^{\pm 1.6}$(\improve{23.3}) & $83.8^{\pm 0.0}$(\improve{31.8}) \\
\bottomrule
\end{tabular}
\end{table}

\subsection{Quantitative Results on NAVSIM}
\label{subsec:navsim_results}

\begin{table*}[t]
\small
\centering
\caption{\textbf{RK4-based framework integration on NAVSIM \texttt{navtest}.} 
We take vision-only, action policies and map their action predictions into waypoints via our framework instantiations: KBM and CCPP.
The \emph{Prediction} column indicates the model's output space (\textit{waypoints} vs.\ \textit{actions}). 
We report NAVSIM \texttt{navtest} six non-reactive open-loop metrics: number of collisions (NC), distance-to-all-collisions (DAC), time-to-collision (TTC), comfort (Comf.), efficiency penalty (EP), and the planning-driven metric score (PDMS) \cite{Dauner24NAVSIM}.
In addition to the CV and Ego status MLP baselines, we report results for \emph{LAW}, the current state-of-the-art vision-only method on NAVSIM \texttt{navtest}.}
\label{tab:navsim_kbm_rk4}
\begin{tabular}{lcccccccc}
\toprule
\textbf{Model} & \textbf{Lifting} & \textbf{Prediction} &
\textbf{NC} $\uparrow$ & \textbf{DAC} $\uparrow$ & \textbf{TTC} $\uparrow$ &
\textbf{Comf.} $\uparrow$ & \textbf{EP} $\uparrow$ & \textbf{PDMS} $\uparrow$ \\
\midrule
\rowcolor{gray!10} CV \cite{Dauner24NAVSIM}   & — & waypoints & $68.0$ & $57.8$ & $50.0$ & $100$ & $19.4$ & $20.6$ \\
\rowcolor{gray!10} Ego MLP \cite{Dauner24NAVSIM}        & — & waypoints & $93.0$ & $77.3$ & $83.6$ & $100$ & $62.8$ & $65.6$ \\
\rowcolor{gray!10} LTF \cite{Chitta22TransFuser} & — & waypoints & $97.4$ & $92.8$ & $92.4$ & $100$ & $79.0$ & $83.8$ \\
\rowcolor{gray!10} UniAD \cite{Hu2022PlanningorientedAD} & — & waypoints & 97.8 & 91.9 & 92.9 & 100 & 78.8 & 83.4 \\
\rowcolor{gray!10} LAW \cite{Li25LatentWorldModel} & — & waypoints & $96.4$ & $95.4$ & $88.7$ & $99.9$ & $81.7$ & $84.6$ \\
\midrule
\multirow{3}{*}{CIL++ \cite{Xiao2023cilpp}} & $\mlplift$ & waypoints & $96.5$ & $94.6$ & $91.8$ & $100$ & $79.3$ & $84.5$ \\
  & $\kbmlift$ & actions & $95.6$ & $94.7$ & $89.8$ & $100$ & $78.3$ & $83.3$ \\
  & $\ccpplift$ & actions & $95.9$ & $93.7$ & $90.1$ & $100$ & $77.4$ & $82.5$ \\
\midrule
\multirow{3}{*}{MILE \cite{Hu22MILE}} & $\mlplift$ & waypoints & $96.7$ & $93.7$ & $91.7$ & $100$ & $78.0$ & $83.5$ \\
 & $\kbmlift$ & actions & $95.5$ & $94.3$ & $89.8$ & $100$ & $77.4$ & $82.7$ \\
 & $\ccpplift$ & actions & $96.1$ & $93.8$ & $90.6$ & $100$ & $77.5$ & $82.8$ \\
\bottomrule
\end{tabular}
\end{table*}

\paragraph{Benchmarks and metrics.} On \texttt{navtest}, a \emph{non-reactive} benchmark where the ego plan and the environment are decoupled (i.e., other agents do not respond to the ego vehicle), open-loop performance is quantified by the \emph{Planning Discrepancy Metric Score} (PDMS) and its component metrics (NC, DAC, TTC, comfort, EP)~\cite{Dauner24NAVSIM}. On \texttt{navhard}, a \emph{reactive} benchmark where background traffic can respond to the ego, evaluation follows the NAVSIM v2 pseudo-simulation protocol with two stages (S1/S2) and is quantified by the \emph{Extended Predictive Driver Model Score} (EPDMS) and its safety/compliance, progress, lane-keeping, and comfort subscores~\cite{Cao25PseudoSimulation}.

\paragraph{Compared methods.}
We assess the performance of our approach by coupling two strong vision-only \emph{action} policies, \textbf{CIL++}~\cite{Xiao2023cilpp} and \textbf{MILE}~\cite{Hu22MILE}, to our differentiable framework instantiations: KBM and CCPP. In all cases, the policies predict an action horizon and our framework maps actions to a $K$-step ego-frame waypoint trajectory for supervision and evaluation. To isolate the benefit of our framework from the benefit of simply predicting in the waypoint space, we also couple the action-based baselines with an MLP that maps predicted actions to waypoints.

\paragraph{Results on \texttt{navtest} (NAVSIM v1).}
Tab.~\ref{tab:navsim_kbm_rk4} reports open-loop \texttt{navtest} performance. Among our action-based variants, CIL++ w/ KBM achieves the highest PDMS ($83.3$), outperforming CIL++ w/ CCPP ($82.5$) and narrowing the gap to strong vision-only waypoint predictors: it is within $1.5\%$ of the state-of-the-art LAW ($84.6$) and within $0.6\%$ of Latent TransFuser ($83.8$).
For MILE, both heads are competitive ($81.0$ with KBM and $82.3$ with CCPP). For both CIL++ and MILE, our MLP controller slightly outperforms our framework instantiations on \texttt{navtest}. Since \texttt{navtest} is non-reactive, a MLP action-to-waypoint controller can slightly outperform our KBM/CCPP instantiations because evaluation reduces to open-loop trajectory scoring under a fixed scene evolution. For further discussion of \texttt{navtest} results, see Appendix~\hyperref[sec:navsim_navtest_additional]{\ref*{sec:navsim_navtest_additional}}.

\begin{table*}[t]
\centering
\small
\setlength{\tabcolsep}{4pt}
\renewcommand{\arraystretch}{1.05}
\caption{\textbf{RK4-based framework integration on NAVSIM \texttt{navhard}.} We report NAVSIM \texttt{navhard} metrics from the NAVSIM v2 pseudo-simulation protocol~\cite{Cao25PseudoSimulation}: the overall extended predictive driver model score (EPDMS) and its subscores, covering safety and compliance (no at-fault collisions, NC; drivable area compliance, DAC; traffic light compliance, TLC), progress (ego progress, EP), collision risk (time to collision, TTC), lane keeping (LK), and comfort (history comfort, HC; extended comfort, EC). For reference, we also include the Constant Velocity, Ego MLP and Latent TransFuser models \cite{Chitta22TransFuser,Cao25PseudoSimulation}.}
\label{tab:navhard_framework_rk4}
\begin{tabular}{c c | c c c | c c c | c c c}
\toprule
\multirow[c]{2}{*}{\textbf{Metric}} &
\multirow[c]{2}{*}{\textbf{Stage}} &
\multirow[c]{2}{*}{\textbf{CV}} &
\multirow[c]{2}{*}{\textbf{Ego MLP}} &
\multirow[c]{2}{*}{\textbf{LTF}} &
\multicolumn{3}{c|}{\textbf{CIL++}~\cite{Xiao2023cilpp}} &
\multicolumn{3}{c}{\textbf{MILE}~\cite{Hu22MILE}} \\
\cmidrule(lr){6-8}\cmidrule(lr){9-11}
& & & & &
\textbf{w/ $\mlplift$} & \textbf{w/ $\kbmlift$} & \textbf{w/ $\ccpplift$} &
\textbf{w/ $\mlplift$} & \textbf{w/ $\kbmlift$} & \textbf{w/ $\ccpplift$} \\
\midrule

\multirow{2}{*}{NC$\uparrow$}  & S1 & 88.8 & 93.2 & \textbf{96.2}
& 94.1 & 95.7 & \textbf{96.2}
& 95.9 & 95.1 & 95.6 \\
                              & S2 & \textbf{83.2} & 77.2 & 77.7
& 76.9 & 76.9 & 79.2
& 79.3 & 80.6 & 79.8 \\
\midrule

\multirow{2}{*}{DAC$\uparrow$} & S1 & 42.8 & 55.7 & 79.5
& 80.0 & \textbf{84.0} & 78.2
& 70.9 & 80.9 & 81.1 \\
                              & S2 & 59.1 & 51.9 & 70.2
& 56.9 & 62.4 & 69.9
& 71.6 & \textbf{75.3} & 73.6 \\
\midrule

\multirow{2}{*}{DDC$\uparrow$} & S1 & 70.6 & 86.6 & \textbf{99.1}
& 97.6 & 97.9 & 97.1
& 98.1 & 97.3 & 96.8 \\
                              & S2 & 76.5 & 74.4 & 84.2
& 76.5 & 82.6 & 83.8
& 85.2 & \textbf{86.6} & 86.2 \\
\midrule

\multirow{2}{*}{TLC$\uparrow$} & S1 & 99.3 & 99.3 & 99.5
& \textbf{99.6} & \textbf{99.6} & \textbf{99.6}
& \textbf{99.6} & \textbf{99.6} & \textbf{99.6} \\
                              & S2 & 98.0 & 98.2 & 98.0
& \textbf{98.5} & 97.8 & 98.2
& 98.1 & \textbf{98.5} & 98.3 \\
\midrule

\multirow{2}{*}{EP$\uparrow$}  & S1 & 77.5 & 81.2 & 84.1
& 85.4 & \textbf{85.5} & 85.2
& 84.3 & 83.7 & 84.1 \\
                              & S2 & 71.3 & 77.1 & 85.1
& 88.5 & 88.7 & 88.1
& 87.9 & 88.5 & \textbf{88.8} \\
\midrule

\multirow{2}{*}{TTC$\uparrow$} & S1 & 87.3 & 92.2 & \textbf{95.1}
& 92.9 & 93.6 & 94.7
& 92.2 & 94.2 & 93.8 \\
                              & S2 & \textbf{81.1} & 75.0 & 75.6
& 73.7 & 73.4 & 76.3
& 75.9 & 76.2 & 76.1 \\
\midrule

\multirow{2}{*}{LK$\uparrow$}  & S1 & 78.6 & 83.5 & 94.2
& 93.6 & 94.9 & 94.7
& 92.9 & \textbf{95.8} & \textbf{95.8} \\
                              & S2 & 47.9 & 40.8 & 45.4
& 46.3 & 48.3 & 49.2
& 49.8 & 47.7 & \textbf{50.9} \\
\midrule

\multirow{2}{*}{HC$\uparrow$}  & S1 & 97.1 & 97.5 & 97.5
& \textbf{97.8} & 97.6 & \textbf{97.8}
& 97.6 & \textbf{97.8} & 97.6 \\
                              & S2 & 97.1 & \textbf{97.8} & 95.7
& 96.3 & 95.7 & 97.0
& 95.9 & 96.4 & 96.9 \\
\midrule

\multirow{2}{*}{EC$\uparrow$}  & S1 & 60.4 & 77.7 & \textbf{79.1}
& 68.9 & 73.8 & 69.8
& 74.7 & 66.7 & 70.7 \\
                              & S2 & 61.9 & \textbf{79.8} & 75.9
& 60.4 & 60.4 & 61.1
& 65.5 & 59.2 & 65.0 \\
\midrule

\multicolumn{2}{c|}{\textbf{EPDMS$\uparrow$}}
& 10.9 & 12.7 & 23.1
& 21.0 & \cellcolor{blue!15}27.3 & \cellcolor{blue!15}27.3
& 24.2 & \cellcolor{blue!30}28.2 & \cellcolor{blue!50}29.5 \\
\bottomrule
\end{tabular}%

\end{table*}

\paragraph{Results on \texttt{navhard} (NAVSIM v2).}
Tab.~\ref{tab:navhard_framework_rk4} shows that integrating our framework yields strong performance in the reactive pseudo-simulation setting for both CIL++ and MILE. The best overall result is obtained by MILE w/ CCPP (EPDMS $29.5$), followed by MILE w/ KBM (EPDMS $28.2$), while both CIL++ instantiations reach EPDMS $27.3$. All framework-coupled variants outperform LTF, the strongest vision-only baseline~\cite{Cao25PseudoSimulation}. In the reactive pseudo-simulation setting of \texttt{navhard}, the MLP controllers fall behind our framework instantiations. This contrasts with \texttt{navtest}, where the non-reactive open-loop protocol allowed MLP controllers to slightly outperform KBM/CCPP by optimizing the fixed-scene trajectory scoring. In \texttt{navhard}, however, the ego plan influences the future scene and errors compound through interaction, so enforcing kinematic structure becomes beneficial. For further discussion on \texttt{navhard} results, see Appendix~\hyperref[sec:navsim_navhard_additional]{\ref*{sec:navsim_navhard_additional}}.

\subsection{Quantitative results on Bench2Drive}
\label{subsec:b2d}

\paragraph{Training dynamics.}
To evaluate the closed-loop performance and training dynamics of our framework, we conduct additional experiments on Bench2Drive~\cite{Jia24Bench2Drive}. Specifically, we evaluate our models on \emph{Dev10}, a validation set of $10$ of the most challenging and diverse routes in the benchmark~\cite{Jia25DriveTransformer}. We report the Driving Score (DS) and Route Completion (RC)~\cite{Jia24Bench2Drive}.

Considering the top-$k$ checkpoints offers an insight into the robustness of the model across training. Tab.~\ref{tab:topk_ds_euler_vs_rk4_euler2rk4_stddec} shows that, while Euler-based KBM and CCPP instantiations already improve Driving Score (DS) over the baselines, replacing Euler with RK4 further stabilizes training dynamics.

\paragraph{RK4 stabilises vehicle-model training dynamics.}
Across both CIL++ and MILE couplings, RK4 consistently reduces variance relative to Euler at fixed Top-$k$, indicating more stable training dynamics. The effect is most pronounced for CCPP, where RK4 yields large dispersion drops for CIL++ (Top-3/5/10: \(\Delta\sigma=5.7/4.0/2.5\)) and also stabilizes MILE (Top-3/5/10: \(\Delta\sigma=2.4/2.4/1.4\)). KBM is already comparatively stable under Euler and still benefits from RK4 with consistent, smaller variance reductions (CIL++: \(\Delta\sigma=1.8/1.4/0.2\); MILE: \(\Delta\sigma=0.6/0.7/0.5\)), while mean DS changes remain modest.

\paragraph{Best-checkpoint evaluation with multiple seeds.}
For each method, we select the best \emph{Dev10} checkpoint and re-evaluate it under three random evaluation seeds; Table~\ref{tab:kbm_dev10_rk4} reports mean\,$\pm$\,standard deviation. Overall, both framework instantiations consistently improve closed-loop performance: for MILE, KBM yields large gains in both DS and RC, while CCPP also improves DS and achieves strong completion; for CIL++, the improvement is even larger. In this RK4 setting, KBM remains the most stable option for CIL++ (lowest DS variance), whereas CCPP attains the highest mean DS for CIL++ while maintaining competitive RC. Compared to the Euler counterparts, these RK4 results exhibit overall lower variance across evaluation seeds (see Appendix~\hyperref[app:extrab2d_euler]{\ref*{app:extrab2d_euler}} for more details).

\subsection{Correlation Study}
\label{subsec:correlation}

\begin{figure}[t]
\centering
\scalebox{0.85}{
\begin{tikzpicture}
\begin{axis}[
    width=\columnwidth,
    height=0.85\columnwidth,
    xlabel={Horizon time (s)},
    xticklabel={
        \pgfmathparse{\tick/2}
        \pgfmathprintnumber{\pgfmathresult}
    },
    ylabel={Pearson correlation $\rho$},
    xmin=1, xmax=40,
    grid=both,
    legend columns=2,
    legend style={
        at={(0.5,1.05)},
        anchor=south,
        font=\footnotesize,
        draw=none,
        fill=none
    },
    cycle list name=color list,
    every axis plot/.append style={thick},
    multilane/.style={blue},
    singlelane/.style={red},
    metric-l1/.style={solid},
    metric-steer/.style={densely dotted},
    metric-kbm/.style={
        solid,
        mark=triangle*,
        mark options={scale=0.8},
        mark repeat=3
    },
    metric-ccpp/.style={
        solid,
        mark=square*,
        mark options={scale=0.7},
        mark repeat=3
    },
]

\addplot[multilane,metric-l1] coordinates {(1,0.687904690266042) (40,0.687904690266042)};
\addlegendentry{MultiLane -- L1}

\addplot[multilane,metric-steer] coordinates {(1,0.051409723116587) (40,0.051409723116587)};
\addlegendentry{MultiLane -- L1 (Steering)}

\addplot[multilane,metric-kbm] coordinates {
(1,-0.005216598987163)
(2,0.006967540413294)
(3,0.01263134514289)
(4,0.010007298406435)
(5,0.005488351217499)
(6,0.000823364528393)
(7,0.002911213404821)
(8,-3.79169228562834e-06)
(9,-0.005946071659687)
(10,-0.013721803825944)
(11,-0.026032106035552)
(12,-0.034754143207639)
(13,-0.035650491822124)
(14,-0.036092891114494)
(15,-0.031499407015168)
(16,-0.024691538444959)
(17,-0.02771486645414)
(18,-0.030064918627983)
(19,-0.034850792045558)
(20,-0.040478655870938)
(21,-0.04676223415791)
(22,-0.046218169989025)
(23,-0.047263189476371)
(24,-0.046408917468633)
(25,-0.047835320683873)
(26,-0.051167533584882)
(27,-0.060448399068618)
(28,-0.069590872662092)
(29,-0.074617844762124)
(30,-0.081504116214178)
(31,-0.084363030045331)
(32,-0.085420229352468)
(33,-0.086588507347248)
(34,-0.08505887839323)
(35,-0.084275132065246)
(36,-0.079520193632641)
(37,-0.077863046293716)
(38,-0.07345782267481)
(39,-0.069714048738683)
(40,-0.064387617346767)
};
\addlegendentry{MultiLane -- L1 KBM}

\addplot[multilane,metric-ccpp] coordinates {
(1,0.178825749125307)
(2,-0.034718258664382)
(3,-0.145806607130445)
(4,-0.185135742414827)
(5,-0.202103389247085)
(6,-0.214706191686161)
(7,-0.212394220872444)
(8,-0.180886189149511)
(9,-0.140406531351992)
(10,-0.113568034859843)
(11,-0.077971616677738)
(12,-0.006101480587095)
(13,0.082943761419688)
(14,0.139188715001293)
(15,0.166698458896956)
(16,0.177401111421421)
(17,0.174988238448637)
(18,0.173463757725916)
(19,0.175971230121805)
(20,0.17716197100517)
(21,0.180081525185208)
(22,0.181139784582724)
(23,0.178071113015293)
(24,0.162713025310973)
(25,0.151845865566486)
(26,0.14399547583628)
(27,0.139088883486228)
(28,0.136336790470824)
(29,0.132840195511387)
(30,0.13526728090796)
(31,0.125363751357052)
(32,0.120914982108508)
(33,0.112943633860526)
(34,0.116058396700669)
(35,0.117970417152507)
(36,0.126620295652144)
(37,0.131017377955558)
(38,0.134066471311974)
(39,0.128866778991597)
(40,0.119689827448825)
};
\addlegendentry{MultiLane -- L1 CCPP}

\addplot[singlelane,metric-l1] coordinates {(1,0.92419223443324) (40,0.92419223443324)};
\addlegendentry{SingleLane -- L1}

\addplot[singlelane,metric-steer] coordinates {(1,-0.660777685096442) (40,-0.660777685096442)};
\addlegendentry{SingleLane -- L1 (Steering)}

\addplot[singlelane,metric-kbm] coordinates {
(1,0.339867031928956)
(2,0.36626959974344)
(3,0.334981174908959)
(4,0.263766263790751)
(5,0.194002279212148)
(6,0.126378017319547)
(7,0.052512161821695)
(8,-0.040670217968943)
(9,-0.120817339863242)
(10,-0.209365744886625)
(11,-0.258980950947761)
(12,-0.338890430428074)
(13,-0.427762188236575)
(14,-0.507434481411884)
(15,-0.553498270940212)
(16,-0.603006118215095)
(17,-0.646214375971647)
(18,-0.685282341630431)
(19,-0.711882234170858)
(20,-0.728343194857859)
(21,-0.739035307246284)
(22,-0.738642226243367)
(23,-0.747203900291887)
(24,-0.742847435662993)
(25,-0.754199231363035)
(26,-0.771695917966103)
(27,-0.776838177617514)
(28,-0.780118074697264)
(29,-0.785128988048583)
(30,-0.794426523341731)
(31,-0.785350003526416)
(32,-0.785589531283719)
(33,-0.797793082100875)
(34,-0.806811023737102)
(35,-0.807662624710743)
(36,-0.803489125007367)
(37,-0.80691095250989)
(38,-0.807949411466418)
(39,-0.800520934341708)
(40,-0.794201165165454)
};
\addlegendentry{SingleLane -- L1 KBM}

\addplot[singlelane,metric-ccpp] coordinates {
(1,0.812169082018743)
(2,0.271653356755063)
(3,-0.471425645052108)
(4,-0.715686132435077)
(5,-0.789888970813034)
(6,-0.781748081174789)
(7,-0.794044983729209)
(8,-0.797913409418601)
(9,-0.748559533313734)
(10,-0.730417931391131)
(11,-0.773122050296923)
(12,-0.778614758679248)
(13,-0.770396069264873)
(14,-0.762320030480872)
(15,-0.712653711842522)
(16,-0.657750423033333)
(17,-0.690349746186384)
(18,-0.659535092838715)
(19,-0.60334765193298)
(20,-0.57652428453659)
(21,-0.584836680812861)
(22,-0.60220429191826)
(23,-0.613282190602624)
(24,-0.655869149368275)
(25,-0.675264598287315)
(26,-0.684466417903294)
(27,-0.705137956192274)
(28,-0.683707130129439)
(29,-0.67611926440198)
(30,-0.688300511208726)
(31,-0.663364354698524)
(32,-0.610814562464182)
(33,-0.591243476589201)
(34,-0.576671142697119)
(35,-0.552088439327964)
(36,-0.57524513140982)
(37,-0.587832758270201)
(38,-0.622190471320797)
(39,-0.626424243446261)
(40,-0.643495693995837)
};
\addlegendentry{SingleLane -- L1 CCPP}

\end{axis}
\end{tikzpicture}
}
\caption{Pearson correlation of different metric errors of different CIL++ models}
\label{fig:correlation}
\end{figure}
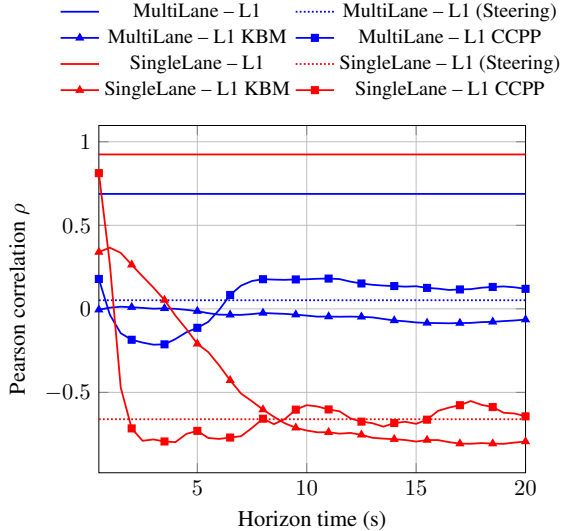

We study how offline action and waypoint errors correlate with closed-loop \emph{Success Rate} (SR). Following the subsequent analysis protocol of CIL++ by Porres \emph{et al.} \yrcite{Porres24GuidingAttention}, we take the CIL++ model trained on each of the two proposed datasets, single-lane and multi-lane \cite{Xiao2023cilpp}. Concretely, we run inference on the corresponding Roach validation split \cite{Porres24GuidingAttention} to collect predicted low-level actions (throttle, steer, brake) together with their ground-truth targets.

Figure~\ref{fig:correlation} reports the Pearson correlation between closed-loop Success Rate (SR) and different error metrics. Since both predictions and ground truth are available only in the action space (throttle, steer, brake), we additionally lift each action sequence to a waypoint trajectory using the corresponding kinematic model (KBM or CCPP), and compute the resulting waypoint-level losses between predicted and ground-truth rollouts. Overall, losses built on our framework rollouts (KBM/CCPP) align better with downstream closed-loop performance than action-space losses (e.g., steering or acceleration-only), because they measure the resulting trajectory error rather than isolated per-step control mismatches. Further details about the correlation study can be found in Appendix~\hyperref[app:extracorrst]{\ref*{app:extracorrst}}

\subsection{Numerical Integration and Error Analysis}
\label{subsec:numerics}

To understand why CCPP with RK4 excels in closed-loop evaluation, we analyze the numerical error introduced by different integration schemes on the validation set.

Figure~\ref{fig:cf8_2hz_xy_error} shows trajectory error for $C_f=8$ at $\Delta t = 0.5$s (2Hz). CCPP consistently maintains lower error throughout the horizon, particularly for distant waypoints where KBM's error grows rapidly. This advantage stems from CCPP's curvature-continuous representation and multiple substeps per interval ($n_{\text{int}}$), which provide better accuracy even at coarse temporal resolution. While larger $\Delta t$ increases error for all models, CCPP degrades more gracefully. At high framerates (10Hz), KBM-Euler achieves comparable error to CCPP-RK4, but this open-loop analysis underestimates error accumulation in closed-loop settings where predictions influence future observations.

These findings explain CCPP-RK4's superior performance in reactive pseudo-simulation (\texttt{navhard}) and closed-loop evaluation (Bench2Drive), where trajectory errors compound over time. The additional computation proves worthwhile when accurate long-horizon prediction is critical. See Appendix~\ref{app:numerical_analysis} for a comprehensive analysis across different horizons and framerates.

\begin{figure}[ht]
    \centering
    \includegraphics[width=0.9\columnwidth]{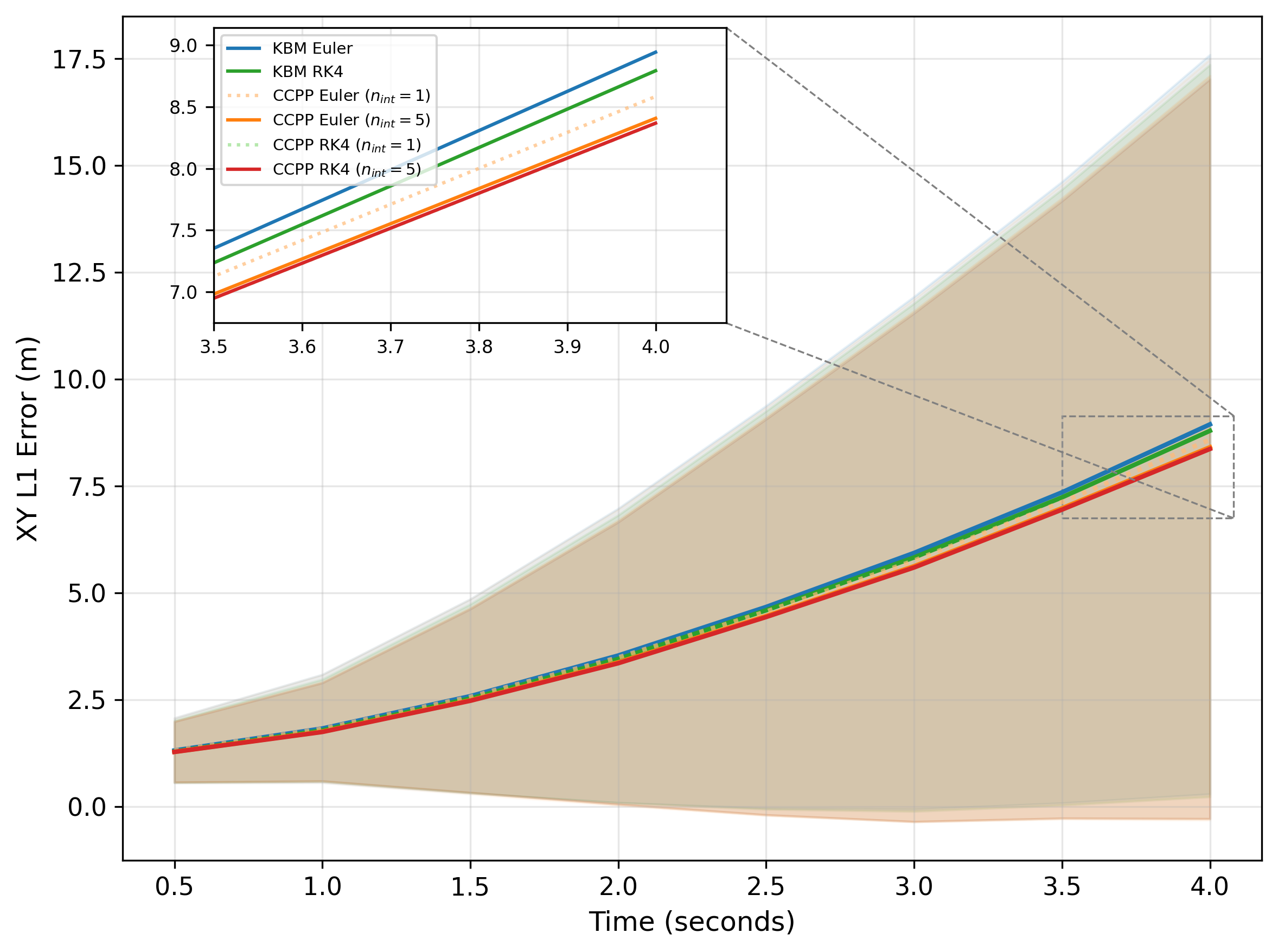}
    \caption{\textbf{Numerical error vs. waypoint.} Average and std. dev. of the $L_1$ waypoint error (meters) for $C_f=8$, $\Delta t = 0.5$s. CCPP maintains lower error throughout the horizon.}
    \label{fig:cf8_2hz_xy_error}
\end{figure}

%% file: sec/5_conclusions.tex
\section{Conclusions}
\label{sec:conclusions}

We proposed a differentiable framework that bridges \emph{action-based} end-to-end driving policies and \emph{waypoint-only} benchmarks by rolling out predicted controls into ego-frame waypoint trajectories. Across NAVSIM and Bench2Drive, both instantiations consistently improve both open-loop and closed-loop performance, as well as training dynamics. Finally, our framework-based offline waypoint $L_1$ error provides a stronger correlation, supporting more reliable checkpoint selection than action-space losses.

\paragraph{Limitations.} Our formulation assumes planar vehicle motion in $\mathbb{R}^2$ and therefore neglects 3D effects such as grade, banking, and suspension dynamics. In addition, our CCPP rollout is initialised with a simple nominal state (typically straight motion with $\kappa_0=0$), which may be suboptimal in scenarios with strong initial curvature or complex manoeuvre onsets.

\section*{Impact Statement}
This paper presents work whose goal is to advance the field of machine learning. There are many potential societal consequences of our work, none of which we feel must be specifically highlighted here.

%% file: sec/6_appendix.tex
\newpage
\appendix

\onecolumn
\begin{center}
{\Large\bfseries APPENDIX}
\end{center}
\vspace{0.5em}

\setcounter{section}{0}

\renewcommand{\thesection}{\Alph{section}}
\renewcommand{\thesubsection}{\thesection\arabic{subsection}}
\renewcommand{\thesubsubsection}{\thesubsection\arabic{subsubsection}}

\noindent
In the appendix, we provide additional details on the proposed method (\hyperref[app:method]{Appendix~A}). We also present a mathematical proof of our framework (\hyperref[app:proof]{Appendix~B}) and further information on our evaluation methodology (\hyperref[app:eval]{Appendix~C}). Finally, we report additional results that could not be fit in the main paper (\hyperref[app:results]{Appendix~D}).

\section{Method Details}
\label{app:method}

\subsection{Continuous kinematic bicycle dynamics.}
\label{app:kbm_method} 

We consider the continuous kinematic bicycle dynamics for the ego vehicle~\cite{Rajamani12VehicleControl,Polack17KinematicBicycle}
\begin{align}
    \dot{x}(\tilde{t})     &= v(\tilde{t})\cos\theta(\tilde{t}), \label{eq:kbm-ct-x}\\
    \dot{y}(\tilde{t})     &= v(\tilde{t})\sin\theta(\tilde{t}), \label{eq:kbm-ct-y}\\
    \dot{\theta}(\tilde{t})&= \frac{v(\tilde{t})}{L}\tan\bigl(\delta(\tilde{t})\bigr), \label{eq:kbm-ct-theta}\\
    \dot{v}(\tilde{t})     &= a^{\text{lon}}(\tilde{t}), \label{eq:kbm-ct-v}
\end{align}
where $(x(\tilde{t}), y(\tilde{t})) \in \mathbb{R}^2$ is the position in the ego frame, $\theta(\tilde{t}) \in \mathbb{R}$ is the yaw angle, and $v(\tilde{t}) \in \mathbb{R}$ is the speed.

In practice, we implement the kinematic bicycle model in discrete time with sampling period $\Delta t > 0$. Let $\tilde{t}_k := k \Delta t$ denote the $k$-th sampling instant and define the sampled kinematic state
\begin{align}
    z_{t,k}
    &:= \bigl(x(\tilde{t}_k), y(\tilde{t}_k), \theta(\tilde{t}_k), v(\tilde{t}_k)\bigr) \nonumber\\
    &= (x_{t,k}, y_{t,k}, \theta_{t,k}, v_{t,k}) \in \mathbb{R}^4,
    \label{eq:kbm_state}
\end{align}
which matches the state variable $z_{t,k}$ used in the rollout recursion~\eqref{eq:next_rollout}.
We anchor the ego frame at decision time $t$ as in Section~\ref{subsec:lifting_operator}, and set the initial state
\begin{equation}
    z_{t,0} = (x_{t,0},y_{t,0},\theta_{t,0},v_{t,0}) = (0,0,0,v_0),
    \label{eq:kbm_init}
\end{equation}
where the lifting input $s_t$ reduces to the initial speed $v_0 \in \mathbb{R}_{\ge 0}$.

\paragraph{Control activation.}
For KBM, the raw action channels are $a_{t,k}=(\tau_{t,k},\sigma_{t,k},\beta_{t,k})$, corresponding to throttle, steering, and brake (Table~\ref{tab:design-space}).
We apply bounded activations to obtain normalised controls
\begin{equation}
\hat{\tau}_{t,k}=\sigma\!\left(\tau_{t,k}\right)\in(0,1),\qquad
\hat{\sigma}_{t,k}=\tanh\!\left(\sigma_{t,k}\right)\in(-1,1),\qquad
\hat{\beta}_{t,k}=\sigma\!\left(\beta_{t,k}\right)\in(0,1),
\label{eq:kbm_decode_norm}
\end{equation}
and map these to physical controls $u_{t,k}=(a^{\text{lon}}_{t,k},\delta_{t,k})\in\mathcal{U}$ by
\begin{equation}
a^{\text{lon}}_{t,k} = a_{\max}\bigl(\hat{\tau}_{t,k}-\hat{\beta}_{t,k}\bigr),
\qquad
\delta_{t,k}=\delta_{\max}\hat{\sigma}_{t,k},
\label{eq:kbm_decode_phys}
\end{equation}
where $a_{\max}>0$ is a longitudinal gain and $\delta_{\max}\in(0,\tfrac{\pi}{2})$ is the maximum steering angle.

\paragraph{Discrete-time rollout: Euler.}
Our Euler implementation uses a semi-implicit update in which the speed and heading are updated first, and the new values are then used to advance position:
\begin{align}
v_{t,k+1} &= v_{t,k} + a^{\text{lon}}_{t,k}\,\Delta t, \label{eq:kbm-euler-v}\\
\theta_{t,k+1} &= \theta_{t,k} + \frac{v_{t,k+1}}{L}\tan(\delta_{t,k})\,\Delta t, \label{eq:kbm-euler-theta}\\
x_{t,k+1} &= x_{t,k} + v_{t,k+1}\cos(\theta_{t,k+1})\,\Delta t, \label{eq:kbm-euler-x}\\
y_{t,k+1} &= y_{t,k} + v_{t,k+1}\sin(\theta_{t,k+1})\,\Delta t, \label{eq:kbm-euler-y}
\end{align}
for $k=0,\dots,C_f-1$, with $z_{t,0}$ as in \eqref{eq:kbm_init}. This corresponds to holding the controls
$(a^{\text{lon}}_{t,k},\delta_{t,k})$ piecewise constant on each interval $[\tilde{t}_k,\tilde{t}_{k+1}]$.

\paragraph{Discrete-time rollout: RK4.}
For higher accuracy we also implement classical fourth-order Runge--Kutta (RK4) with piecewise-constant controls over each interval.
Define the continuous-time right-hand side associated with \eqref{eq:kbm-ct-x}--\eqref{eq:kbm-ct-v} as
\begin{equation}
\dot{z} = g(z; a^{\text{lon}},\delta)
=
\begin{pmatrix}
v\cos\theta\\
v\sin\theta\\
\frac{v}{L}\tan(\delta)\\
a^{\text{lon}}
\end{pmatrix},
\qquad
z=(x,y,\theta,v)^\top.
\label{eq:kbm_rhs}
\end{equation}
With $u_{t,k}=(a^{\text{lon}}_{t,k},\delta_{t,k})$ fixed on $[\tilde{t}_k,\tilde{t}_{k+1}]$, RK4 updates
\begin{align}
k_1 &= g\!\left(z_{t,k};\,a^{\text{lon}}_{t,k},\delta_{t,k}\right), \nonumber\\
k_2 &= g\!\left(z_{t,k}+\tfrac{\Delta t}{2}k_1;\,a^{\text{lon}}_{t,k},\delta_{t,k}\right), \nonumber\\
k_3 &= g\!\left(z_{t,k}+\tfrac{\Delta t}{2}k_2;\,a^{\text{lon}}_{t,k},\delta_{t,k}\right), \nonumber\\
k_4 &= g\!\left(z_{t,k}+\Delta t\,k_3;\,a^{\text{lon}}_{t,k},\delta_{t,k}\right), \nonumber\\
z_{t,k+1} &= z_{t,k} + \frac{\Delta t}{6}\left(k_1+2k_2+2k_3+k_4\right),
\label{eq:kbm-rk4}
\end{align}
for $k=0,\dots,C_f-1$, starting from $z_{t,0}$ in \eqref{eq:kbm_init}.

\paragraph{Pose projection.}
To match the action-to-waypoint lifting interface in Section~\ref{subsec:lifting_operator}, we extract planar waypoints via
\begin{equation}
w_{t,k}=h(z_{t,k})=(x_{t,k},y_{t,k})\in\mathbb{R}^2,
\qquad k=1,\dots,C_f,
\label{eq:kbm_waypoint_readout}
\end{equation}
so that $\mathcal{F}_\phi(\mathbf{a}_t,s_t)=\bigl(h(z_{t,1}),\dots,h(z_{t,C_f})\bigr)$ as in \eqref{eq:lifting_explicit}.

\subsection{CCPP arc-length increments \& time discretisation.}
\label{app:ccpp_method} 

We anchor the ego frame at decision time $t$ as in Section~\ref{subsec:lifting_operator} and initialise
\begin{equation}
    x_{t,0}=0,\ y_{t,0}=0,\ \theta_{t,0}=0,\ \kappa_{t,0}=\kappa_0,\ v_{t,0}=v_0,
    \label{eq:ccpp_init}
\end{equation}
where $v_0\in\mathbb{R}_{\ge 0}$ is the initial speed and $\kappa_0\in\mathbb{R}$ the initial curvature (typically $\kappa_0=0$).
The waypoint readout is the pose component
\begin{equation}
    w_{t,k} := (x_{t,k},y_{t,k},\theta_{t,k})\in\mathbb{R}^3,
    \qquad k=1,\dots,C_f,
    \label{eq:ccpp_waypoint_readout}
\end{equation}
and the predicted waypoint sequence is $\mathbf{w}_t=\{w_{t,k}\}_{k=1}^{C_f}\in(\mathbb{R}^3)^{C_f}$.

\paragraph{Arc-length dynamics.}
CCPP parameterises lateral motion by the travelled arc-length $s$ and evolves the pose and curvature according to
\begin{align}
\frac{dx(s)}{ds} &= \cos(\theta(s)), \label{eq:dxds_ccpp}\\
\frac{dy(s)}{ds} &= \sin(\theta(s)), \label{eq:dyds_ccpp}\\
\frac{d\theta(s)}{ds} &= \kappa(s), \label{eq:dthetads_ccpp}\\
\frac{d\kappa(s)}{ds} &= \varsigma(s), \label{eq:dkds_ccpp}
\end{align}
where $\kappa(s)$ is curvature and $\varsigma(s)$ is sharpness (curvature rate).

\paragraph{Control activation.}
For CCPP, the raw action channels are $a_{t,k}=(\tau_{t,k},\xi_{t,k},\beta_{t,k})$, corresponding to throttle, sharpness, and brake (Table~\ref{tab:design-space}).
We apply bounded activations to obtain normalised controls
\begin{equation}
\hat{\tau}_{t,k}=\sigma\!\left(\tau_{t,k}\right)\in(0,1),\qquad
\hat{\xi}_{t,k}=\tanh\!\left(\xi_{t,k}\right)\in(-1,1),\qquad
\hat{\beta}_{t,k}=\sigma\!\left(\beta_{t,k}\right)\in(0,1),
\label{eq:ccpp_decode_norm}
\end{equation}
and map these to physical controls $u_{t,k}=(a^{\text{lon}}_{t,k},\varsigma_{t,k})\in\mathcal{U}$ by
\begin{equation}
a^{\text{lon}}_{t,k} = a_{\mathrm{gain}}\bigl(\hat{\tau}_{t,k}-\hat{\beta}_{t,k}\bigr),
\qquad
\varsigma_{t,k}=\varsigma_M \hat{\xi}_{t,k},
\label{eq:ccpp_decode_phys}
\end{equation}
where $a_{\mathrm{gain}}>0$ is a longitudinal gain and $\varsigma_M>0$ bounds the sharpness.

\paragraph{Time discretisation and arc-length increments.}
Although the lateral dynamics are expressed in arc-length $s$, the policy is executed in discrete time with step $\Delta t>0$.
Let $k=0,\dots,C_f-1$ index the horizon steps.
We maintain a scalar speed variable $v_{t,k}$ (m/s) and update it using the longitudinal acceleration control:
\begin{equation}
    v_{t,k+1} = \max\{v_{t,k} + a_{t,k}^{\text{lon}}\,\Delta t, 0\},
\label{eq:ccpp_speed_update}
\end{equation}
i.e.\ we clamp to nonnegative speeds (forward motion only).
We then approximate the arc length travelled during the $k$-th time step by
\begin{equation}
    \Delta s_{t,k} \approx v_{t,k+1}\,\Delta t.
    \label{eq:ccpp_ds_total}
\end{equation}
To integrate \eqref{eq:dxds_ccpp}--\eqref{eq:dkds_ccpp} over the increment $\Delta s_{t,k}$, we use $n_{\text{int}}\in\mathbb{N}$ uniform substeps
\begin{equation}
\delta s_{t,k} := \frac{\Delta s_{t,k}}{n_{\text{int}}}.
\label{eq:ccpp_ds_step}
\end{equation}
Within each time step $k$, the controls $(a^{\text{lon}}_{t,k},\varsigma_{t,k})$ are held piecewise constant, and the arc-length dynamics are integrated over
$s\in[0,\Delta s_{t,k}]$ using either Euler or RK4 in $s$.

\paragraph{Arc-length rollout: Euler substepping.}
Let $(x^{(0)}_{t,k},y^{(0)}_{t,k},\theta^{(0)}_{t,k},\kappa^{(0)}_{t,k}) := (x_{t,k},y_{t,k},\theta_{t,k},\kappa_{t,k})$ denote the substep initialisation.
For substeps $j=0,\dots,n_{\text{int}}-1$, forward Euler in arc-length yields
\begin{align}
\kappa^{(j+1)}_{t,k} &= \mathrm{clip}\!\left(\kappa^{(j)}_{t,k} + \varsigma_{t,k}\,\delta s_{t,k};\, -\kappa_M,\kappa_M\right), \label{eq:ccpp_euler_kappa}\\
\theta^{(j+1)}_{t,k} &= \theta^{(j)}_{t,k} + \kappa^{(j+1)}_{t,k}\,\delta s_{t,k}, \label{eq:ccpp_euler_theta}\\
x^{(j+1)}_{t,k} &= x^{(j)}_{t,k} + \cos\!\left(\theta^{(j+1)}_{t,k}\right)\,\delta s_{t,k}, \label{eq:ccpp_euler_x}\\
y^{(j+1)}_{t,k} &= y^{(j)}_{t,k} + \sin\!\left(\theta^{(j+1)}_{t,k}\right)\,\delta s_{t,k}, \label{eq:ccpp_euler_y}
\end{align}
where $\mathrm{clip}(x;-\kappa_M,\kappa_M)=\min\{\max\{x,-\kappa_M\},\kappa_M\}$ enforces the curvature bound $|\kappa|\le\kappa_M$.
The discrete-time state is then set to the terminal substep values
\begin{equation}
(x_{t,k+1},y_{t,k+1},\theta_{t,k+1},\kappa_{t,k+1})
:=
(x^{(n_{\text{int}})}_{t,k},y^{(n_{\text{int}})}_{t,k},\theta^{(n_{\text{int}})}_{t,k},\kappa^{(n_{\text{int}})}_{t,k}).
\label{eq:ccpp_euler_terminal}
\end{equation}

\paragraph{Arc-length rollout: RK4 substepping.}
For RK4, we apply classical Runge--Kutta to the arc-length ODE system \eqref{eq:dxds_ccpp}--\eqref{eq:dkds_ccpp} over each substep $\delta s_{t,k}$, with
$\varsigma_{t,k}$ held constant across the substep. Define the arc-length vector field
\begin{equation}
\frac{d}{ds}
\begin{pmatrix}
x\\y\\\theta\\\kappa
\end{pmatrix}
=
F\!\left(x,y,\theta,\kappa;\varsigma\right)
=
\begin{pmatrix}
\cos\theta\\
\sin\theta\\
\kappa\\
\varsigma
\end{pmatrix}.
\label{eq:ccpp_rhs}
\end{equation}
Starting from $(x^{(0)}_{t,k},y^{(0)}_{t,k},\theta^{(0)}_{t,k},\kappa^{(0)}_{t,k})=(x_{t,k},y_{t,k},\theta_{t,k},\kappa_{t,k})$, for each substep $j$
we compute
\begin{align}
K_1 &= F\!\left(x^{(j)}_{t,k},y^{(j)}_{t,k},\theta^{(j)}_{t,k},\kappa^{(j)}_{t,k};\varsigma_{t,k}\right), \nonumber\\
K_2 &= F\!\left(\mathbf{z}^{(j)}_{t,k}+\tfrac{\delta s_{t,k}}{2}K_1;\varsigma_{t,k}\right), \nonumber\\
K_3 &= F\!\left(\mathbf{z}^{(j)}_{t,k}+\tfrac{\delta s_{t,k}}{2}K_2;\varsigma_{t,k}\right), \nonumber\\
K_4 &= F\!\left(\mathbf{z}^{(j)}_{t,k}+\delta s_{t,k}K_3;\varsigma_{t,k}\right), \nonumber\\
\mathbf{z}^{(j+1)}_{t,k}
&=
\mathbf{z}^{(j)}_{t,k}+\frac{\delta s_{t,k}}{6}\left(K_1+2K_2+2K_3+K_4\right),
\label{eq:ccpp_rk4}
\end{align}
where $\mathbf{z}^{(j)}_{t,k}:=(x^{(j)}_{t,k},y^{(j)}_{t,k},\theta^{(j)}_{t,k},\kappa^{(j)}_{t,k})^\top$.
After each RK4 substep we enforce the curvature bound by clamping
\begin{equation}
\kappa^{(j+1)}_{t,k}\leftarrow \mathrm{clip}\!\left(\kappa^{(j+1)}_{t,k};-\kappa_M,\kappa_M\right).
\label{eq:ccpp_rk4_clamp}
\end{equation}
As in \eqref{eq:ccpp_euler_terminal}, the discrete-time state $(x_{t,k+1},y_{t,k+1},\theta_{t,k+1},\kappa_{t,k+1})$ is obtained from the terminal substep state.

\paragraph{Pose projection.}
To match the action-to-waypoint lifting interface in Section~\ref{subsec:lifting_operator}, we may extract planar waypoints via the pose projection
\begin{equation}
\hat{w}_{t,k}=h(z_{t,k})=(x_{t,k},y_{t,k})\in\mathbb{R}^2,
\qquad k=1,\dots,C_f,
\label{eq:ccpp_pose_projection}
\end{equation}
while \eqref{eq:ccpp_waypoint_readout} additionally records heading for analysis. In either case, the lifting output is formed as in \eqref{eq:lifting_explicit}.

\section{Mathematical Proof of Proposition~\ref{prop:lifting-smooth}}
\label{app:proof}

We provide a detailed proof of Proposition~\ref{prop:lifting-smooth}. Throughout, we fix the decision time $t\in\mathbb{N}$ and the control horizon length $C_f\in\mathbb{N}$, and we use the notation of Section~\ref{sec:newmethod}. In particular, the action sequence is
\[
\mathbf{a}_t=(a_{t,0},a_{t,1},\dots,a_{t,C_f-1})\in\mathbb{R}^{C_f\times d_u},
\]
and the lifting operator $\mathcal{F}_\phi$ is defined via the recursion \eqref{eq:next_rollout} and the readout \eqref{eq:lifting_explicit}.

\begin{proof}[Proof of Proposition~\ref{prop:lifting-smooth}]
Fix an initial state $s_t\in\mathcal{S}$ and parameters $\phi$. Define the rollout recursion by
\begin{equation}
z_{t,0} = z_{t,0}(s_t)\in\mathcal{Z},
\qquad
z_{t,k+1} = f_\phi\!\left(z_{t,k},\,\psi_\phi(a_{t,k}),\,\Delta t\right),
\quad k=0,\dots,C_f-1,
\label{eq:app_rollout_recursion}
\end{equation}
and define the predicted waypoint sequence by
\begin{equation}
\mathcal{F}_\phi(\mathbf{a}_t,s_t)
=
\left(h(z_{t,1}),\,h(z_{t,2}),\,\dots,\,h(z_{t,C_f})\right)
\in\mathbb{R}^{C_f\times 2}.
\label{eq:app_lifting_def}
\end{equation}
We prove determinism and $C^1$ regularity. The proof is organised by instantiation: first KBM, then CCPP.

\paragraph{(A) KBM lifting $\kbmlift$.}
In this instantiation, the state has the form $z_{t,k}=(x_{t,k},y_{t,k},\theta_{t,k},v_{t,k})$, the control activation $\psi_\phi$ maps raw actions to bounded physical controls, the dynamics rollout $f_\phi$ corresponds to a numerical discretisation (Euler/RK4) of the continuous bicycle dynamics \eqref{eq:kbm-ct-x}--\eqref{eq:kbm-ct-v}, and the pose projection is $h(x,y,\theta,v)=(x,y)$ (Table~\ref{tab:design-space}).

\smallskip
\noindent\emph{(A.i) Determinism.}
Assumption (a) implies that for each $k$, the activated control
\[
u_{t,k}:=\psi_\phi(a_{t,k})\in\mathcal{U}
\]
is uniquely determined by $a_{t,k}$. Fixing $s_t$ fixes the initial state $z_{t,0}=z_{t,0}(s_t)$ (with ego-frame anchoring ensuring $h(z_{t,0})=(0,0)$).
Assumption (b) implies that the one-step update map
\[
(z,u)\mapsto f_\phi(z,u,\Delta t)
\]
is single-valued, hence the recursion \eqref{eq:app_rollout_recursion} uniquely determines $z_{t,1}$ from $(z_{t,0},u_{t,0})$, then uniquely determines
$z_{t,2}$ from $(z_{t,1},u_{t,1})$, and so on. Formally, by induction on $k$:
\begin{itemize}
\item \emph{Base case:} $z_{t,0}$ is fixed by $s_t$.
\item \emph{Inductive step:} if $z_{t,k}$ is uniquely determined, then $z_{t,k+1}=f_\phi(z_{t,k},\psi_\phi(a_{t,k}),\Delta t)$ is uniquely determined.
\end{itemize}
Thus the state sequence $\{z_{t,k}\}_{k=0}^{C_f}$ is unique, and therefore each waypoint $h(z_{t,k})$ is unique. Hence $\mathcal{F}_\phi(\cdot,s_t)$ is deterministic.

\smallskip
\noindent\emph{(A.ii) $C^1$ regularity.}
Define the \emph{one-step lifted update} map
\[
T_\phi:\mathcal{Z}\times\mathbb{R}^{d_u}\to\mathcal{Z},
\qquad
T_\phi(z,a):= f_\phi\!\left(z,\,\psi_\phi(a),\,\Delta t\right).
\]
By assumptions (a) and (b), $\psi_\phi$ is $C^1$ and $f_\phi$ is $C^1$; therefore $T_\phi$ is $C^1$ as a composition of $C^1$ maps.

We now show by induction on $k$ that $z_{t,k}$ is a $C^1$ function of the partial action sequence $(a_{t,0},\dots,a_{t,k-1})$.
Let $\mathbf{a}_{t,0:k-1}$ denote this partial sequence.

\emph{Base case ($k=0$).} $z_{t,0}$ is constant with respect to $\mathbf{a}_{t,0:-1}$, hence $C^1$.

\emph{Inductive step.} Assume $z_{t,k}=Z_k(\mathbf{a}_{t,0:k-1})$ for some $C^1$ map $Z_k$.
Define
\[
Z_{k+1}(\mathbf{a}_{t,0:k})
:=
T_\phi\!\bigl(Z_k(\mathbf{a}_{t,0:k-1}),\,a_{t,k}\bigr).
\]
Since $Z_k$ and $T_\phi$ are $C^1$, $Z_{k+1}$ is $C^1$ by closure of $C^1$ under composition. Hence $z_{t,k}$ is $C^1$ in $\mathbf{a}_{t,0:k-1}$ for all $k$.

Finally, by assumption (c), $h$ is $C^1$, so each component map
\[
\mathbf{a}_t \mapsto h(z_{t,k}(\mathbf{a}_t))
\]
is $C^1$. The lifting operator \eqref{eq:app_lifting_def} is a finite concatenation of these $C^1$ components, hence
\[
\mathcal{F}_\phi(\cdot,s_t)\in C^1\!\left(\mathbb{R}^{C_f\times d_u},\,\mathbb{R}^{C_f\times 2}\right).
\]

\paragraph{(B) CCPP lifting $\ccpplift$.}
In this instantiation, the state has the form $z_{t,k}=(x_{t,k},y_{t,k},\theta_{t,k},\kappa_{t,k})$, the control activation $\psi_\phi$ maps raw actions to bounded controls (including sharpness $\varsigma$), the rollout $f_\phi$ corresponds to a numerical discretisation (Euler/RK4, with $n_{\text{int}}$ substeps per $\Delta t$) of the continuous arc-length dynamics \eqref{eq:dxds_ccpp}-\eqref{eq:dkds_ccpp}, and the pose projection is $h(x,y,\theta,\kappa)=(x,y)$ (Table~\ref{tab:design-space}). Saturations enforce $v\ge 0$ and $|\kappa|\le \kappa_M$ (Table~\ref{tab:design-space}).

\smallskip
\noindent\emph{(B.i) Determinism.}
The determinism argument is identical in structure to the KBM case and uses only that the CCPP recursion is single-valued.
Assumption (a) ensures that for each $k$ the activated control $u_{t,k}=\psi_\phi(a_{t,k})$ is uniquely determined.
Fixing $s_t$ fixes the initial state $z_{t,0}=z_{t,0}(s_t)$ (with $h(z_{t,0})=(0,0)$ by ego-frame anchoring).
Assumption (b) ensures the update $z_{t,k+1}=f_\phi(z_{t,k},u_{t,k},\Delta t)$ is uniquely determined from $(z_{t,k},u_{t,k})$.
Thus, by the same induction on $k$ as above, the state sequence $\{z_{t,k}\}_{k=0}^{C_f}$ is unique, and hence
$\mathcal{F}_\phi(\cdot,s_t)$ is deterministic.

\smallskip
\noindent\emph{(B.ii) $C^1$ regularity under the standing assumptions.}
Define again
\[
T_\phi(z,a):= f_\phi\!\left(z,\,\psi_\phi(a),\,\Delta t\right).
\]
By assumptions (a) and (b), $T_\phi$ is $C^1$. The same induction as in (A.ii) yields that each $z_{t,k}$ is a $C^1$ function of the partial
action sequence $\mathbf{a}_{t,0:k-1}$, and then assumption (c) implies each waypoint $h(z_{t,k})$ is $C^1$ in $\mathbf{a}_t$.
Finite concatenation gives
\[
\mathcal{F}_\phi(\cdot,s_t)\in C^1\!\left(\mathbb{R}^{C_f\times d_u},\,\mathbb{R}^{C_f\times 2}\right).
\]
This establishes the claimed differentiability for the CCPP lifting operator under the stated hypotheses of Proposition~\ref{prop:lifting-smooth}.

\smallskip
\noindent\emph{Remark (connection to saturations in Table~\ref{tab:design-space}).}
Table~\ref{tab:design-space} notes that CCPP may employ saturations enforcing $v\ge 0$ and $|\kappa|\le \kappa_M$.
Such saturations are typically implemented via pointwise clamping operators, which are not $C^1$ at the switching boundaries.
On any open set of action sequences for which the saturations remain inactive throughout the finite rollout, the effective rollout map coincides locally with a smooth map,
and the above $C^1$ argument applies without modification.
\end{proof}

\section{Training \& Evaluation Details}
\label{app:eval}

\subsection{Additional benchmark information}
\label{app:extrabeninfo}

In \textbf{NAVSIM} \textit{navtest}, the evaluation is \emph{non-reactive}: surrounding traffic does not interact with the ego policy, which is assessed in an open-loop manner by predicting a fixed-horizon trajectory of $K$ future waypoints in the ego frame. In contrast, \textbf{NAVSIM} \textit{navhard} introduces \emph{pseudo-simulation} to better stress robustness under compounding errors while retaining data-driven scalability: it (i) generates \emph{synthetic} observations for perturbed ego futures using a pre-rendered scene representation via 3D Gaussian Splatting, and (ii) makes the environment \emph{reactive}, thereby exposing recovery behaviour under deviations in a setting closer to closed-loop evaluation. Crucially for our purposes, NAVSIM benchmarks are \emph{waypoint-only}: agents must output a fixed-horizon trajectory of ego-frame waypoints. Direct low-level control command execution is not supported \cite{Dauner24NAVSIM,Cao25PseudoSimulation}. This design makes NAVSIM a natural testbed for our framework: by coupling our framework to a control-based policy, action predictions are mapped into waypoint sequences, enabling training and evaluation of control policies without any changes to the benchmark.


We further validate our approach in CARLA~\cite{Dosovitskiy17}, selecting two additional benchmarks: (i) \textbf{Bench2Drive} provides an official expert \emph{base} dataset with $2$M fully annotated frames collected from $13{,}638$ short clips, uniformly covering $44$ interactive scenarios (e.g., cut-ins, overtaking, detours), $23$ weathers and $12$ towns, sampled at a frequency of $10$\,Hz \cite{Jia24Bench2Drive}. Bench2Drive is a closed-loop benchmark built on CARLA that natively supports both action-based and waypoint-based policies~\cite{Jia24Bench2Drive}; (ii) we use a Roach dataset, comprising approximately $10$ hours of \textbf{general driving in Town05} ($356$,$508$ frames at $10 \text{Hz}$) under four weather conditions (ClearNoon, ClearSunset, HardRainNoon, WetNoon)~\cite{Zhang21RLICoach,Porres24GuidingAttention}.

\subsection{Additional implementation details}
\label{app:extraimpdet}

Both NAVSIM and Bench2Drive experiments are implemented in PyTorch Lightning for training. In both benchmarks, we train all models for $100$ epochs with the Adam optimizer (learning rate $1\times 10^{-4}$, weight decay $10^{-7}$) and a step learning-rate schedule that halves the learning rate every $30$ epochs. We use a batch size of $64$ and validate every $10$ epochs. Training runs on $4$ NVIDIA L40S GPUs with distributed data-parallel (DDP) strategy. For Bench2Drive we use the official \emph{base} dataset \cite{Jia24Bench2Drive} for training. For NAVSIM experiments, we cache pre-processed training samples on disk to amortise the cost of scene filtering and trajectory extraction across runs. The training setup for Roach-based experiments is specified in Porres \emph{et al.} \yrcite{Porres24GuidingAttention}.

\paragraph{MLP controllers for waypoint prediction (CIL++/MILE w/$\mlplift$).}
As explained in the main paper, to isolate the effect of output-space alignment from the vehicle-model structure introduced by KBM/CCPP, we implement an additional \emph{waypoint regression} variant for each action-based backbone by replacing the original action head with a lightweight MLP that directly predicts a $K{=}8$ step ego-frame waypoint sequence with $(x,y,\psi)$ per step (i.e., output shape $[B,8,3]$). Importantly, these MLP controllers share the full perception and temporal backbone with the corresponding action policy; the only change is the final prediction head and its training target.

\paragraph{CIL++ w/$\mlplift$.}
Starting from the standard CIL++ pipeline \cite{Xiao2023cilpp}, we swap the original action head for an MLP waypoint head implemented with the same fully-connected block used elsewhere in CIL++ (a feed-forward network that outputs $8\times 3$ scalars and reshapes them to $[B,8,3]$). All upstream components (positional encoding choice, measurement embeddings, Transformer configuration, and pooling) are kept identical to the default CIL++ implementation; the head is initialised with Xavier weights and trained under the same optimisation schedule as the other variants.

\paragraph{MILE w/$\mlplift$.}
For MILE \cite{Hu22MILE}, we similarly attach a simple waypoint regression head \texttt{WaypointMLPHead}: 3-layer MLP with ReLU activations and hidden width 256, to the model's BEV/state embedding. Concretely, after encoding the input history into an embedding sequence, we take the last-step embedding and regress the full future waypoint horizon in a single forward pass, producing $[B,8,3]$. This preserves MILE's perception stack (image encoder, BEV backbone, and measurement/speed conditioning), while removing the need to output low-level controls for this baseline; the rest of the architecture and training protocol is unchanged.

\paragraph{Coupling and supervision.}
Both MLP controllers are trained and evaluated using the same waypoint targets and horizon as our framework-based variants, ensuring that any differences against KBM/CCPP are attributable to (i) unconstrained direct waypoint regression versus (ii) waypoint generation through a differentiable vehicle-model prior.

\section{Additional Results}
\label{app:results}

\subsection{NAVSIM qualitative results}
\label{sec:navsimqual}

To complement the reported quantitative NAVSIM \texttt{navtest} results, we present a qualitative comparison between our framework and two baselines: a constant-velocity agent and the LTF baseline \cite{Cao25PseudoSimulation}. Figure~\ref{fig:navsim_scenes_combined} visualises trajectory predictions for the constant-velocity agent and our four model instantiations across four representative NAVSIM scenes (one scene per row). Figure~\ref{fig:ltf_1row_4scenes} shows the corresponding LTF predictions for the same four examples. In each panel, the ego vehicle is shown in red, \emph{ground-truth (GT) waypoints} are plotted in green, and \emph{predicted waypoints} are plotted in red.

Overall, the constant-velocity baseline frequently fails in challenging scenarios, while both LTF and our framework generate plausible trajectories that follow the intended route. Since LTF and our instantiations can be difficult to distinguish purely from waypoint overlays, we additionally report the \emph{per-waypoint L1 error} comparison between LTF and our framework (mean $\pm$ standard deviation over the four instantiations) for the four scenes in Figure~\ref{fig:scene_mosaic_ltf_vs_ours_big}. This error-based view helps disambiguate subtle differences in the late-horizon waypoints and highlights where one method is systematically more accurate.

Across all four scenes, the constant-velocity agent exhibits the expected limitations of a purely kinematic rollout that does not account for road geometry. In \textbf{Scene 1} (first row of Figure~\ref{fig:navsim_scenes_combined}), the ego approaches a region with traffic: the constant-velocity baseline effectively gets stuck and does not meaningfully progress through the scene, whereas all learned agents (our four instantiations and LTF in Figure~\ref{fig:ltf_1row_4scenes}) continue with reasonable forward motion. In \textbf{Scene 2}, where the ego must execute a left turn at an intersection, the constant-velocity baseline fails in the most direct way by continuing straight rather than following the turning lane. Similar shortcomings appear in \textbf{Scene 3}, which requires a lane change: the constant-velocity trajectory cannot represent the lateral manoeuvre and therefore departs from the GT. Finally, in \textbf{Scene 4}, the baseline provides a plausible early-horizon motion but still lacks the context to anticipate the correct intersection behaviour reliably. These qualitative failures reinforce that constant-velocity rollouts are not a meaningful behavioural prior for NAVSIM \texttt{navtest} beyond serving as a simple lower-bound baseline.

In \textbf{Scenes 1 and 2}, the predicted waypoint overlays of LTF and our instantiations are visually close, and the qualitative assessment alone can be inconclusive. We therefore use Figure~\ref{fig:scene_mosaic_ltf_vs_ours_big} to compare the waypoint-level L1 errors. For \textbf{Scene 1}, our framework achieves a slightly lower mean L1 error than LTF. Notably, the standard deviation across our four instantiations is extremely small---so small that it is barely visible in the plot---suggesting consistent behaviour across random seeds for this scenario. For \textbf{Scene 2}, both LTF and our framework again show largely indistinguishable behaviour at the trajectory level. A closer look at Figure~\ref{fig:navsim_scenes_combined} indicates that LTF and the CIL++ w/KBM instantiation exhibit slightly larger deviations in the last few predicted waypoints, while the remaining instantiations track the GT more tightly; however, these differences are minor. Consistent with that observation, Figure~\ref{fig:scene_mosaic_ltf_vs_ours_big} shows that the overall error profiles are essentially identical at a qualitative level for this scene.

\textbf{Scene 3} provides a clearer separation. Here the ego must perform a lane change, and the learned models are required to predict both the correct lateral shift and the correct progression along the lane. In Figure~\ref{fig:navsim_scenes_combined}, our instantiations visibly track the GT trajectory more closely than LTF, particularly in the late-horizon waypoints where LTF exhibits a small but systematic offset. This qualitative difference is reflected in Figure~\ref{fig:scene_mosaic_ltf_vs_ours_big}: our framework attains a noticeably lower L1 error across the horizon compared to LTF, indicating that the visual improvement is not merely cosmetic but corresponds to a real reduction in waypoint deviation.

\textbf{Scene 4} illustrates a failure mode for our framework. The ego must turn right at an intersection, and our instantiations struggle to predict the right-turn curvature reliably. As shown in Figure~\ref{fig:scene_mosaic_ltf_vs_ours_big}, the early-horizon waypoints (roughly the first three) can incur L1 errors that are higher than those of the constant-velocity baseline, which is a strong indicator of a systematic mismatch rather than random noise. Moreover, the standard deviation across our four instantiations is relatively large for this scene, suggesting sensitivity to initialization or training variance under this particular manoeuvre. In contrast, LTF maintains a consistently lower L1 error than our instantiation mean throughout the horizon, indicating that LTF is more stable and accurate on this right-turn scenario.

\begin{figure*}[p]
    \centering
    \small
    \makebox[0.18\textwidth][c]{\textbf{Constant Velocity}}%
    \hspace{2pt}%
    \makebox[0.18\textwidth][c]{\textbf{CIL++ w/$\kbmlift$}}%
    \hspace{2pt}%
    \makebox[0.18\textwidth][c]{\textbf{CIL++ w/$\ccpplift$}}%
    \hspace{2pt}%
    \makebox[0.18\textwidth][c]{\textbf{MILE w/$\kbmlift$}}%
    \hspace{2pt}%
    \makebox[0.18\textwidth][c]{\textbf{MILE w/$\ccpplift$}}\\
    
    \vspace{2pt}
    
    \adtframe{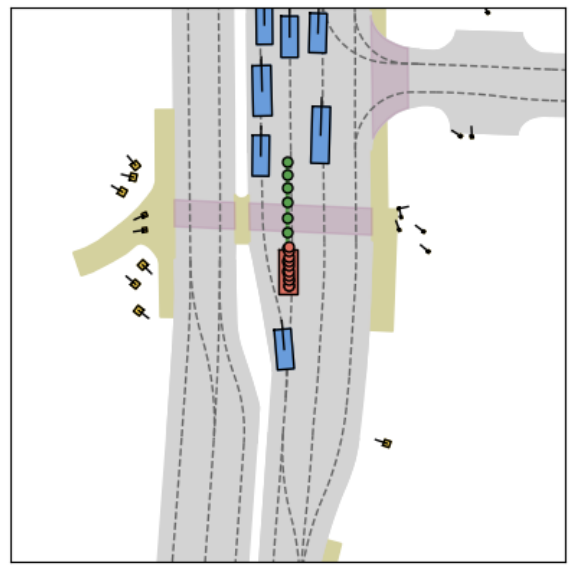}%
    \hspace{2pt}
    \adtframe{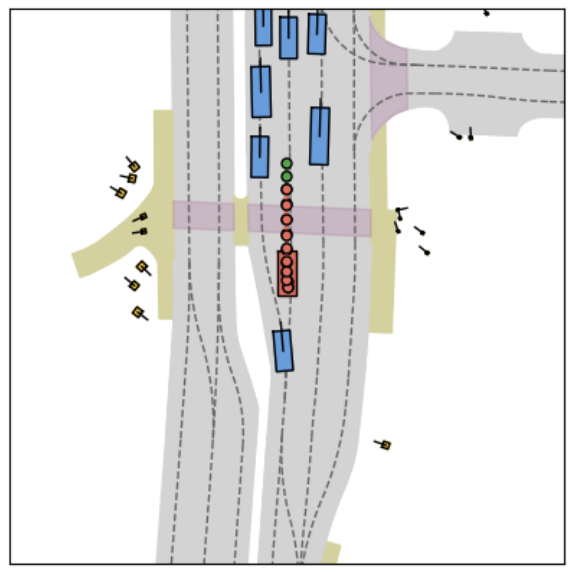}%
    \hspace{2pt}
    \adtframe{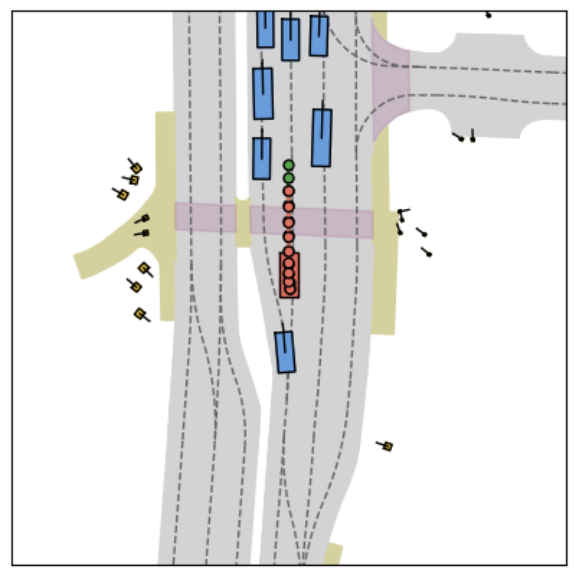}%
    \hspace{2pt}
    \adtframe{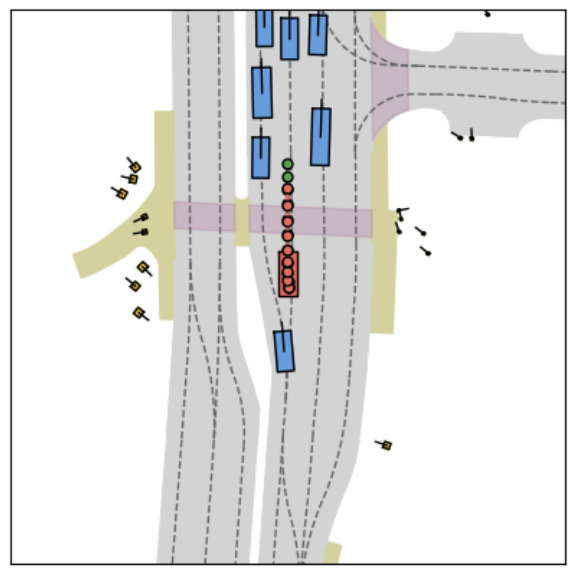}%
    \hspace{2pt}
    \adtframe{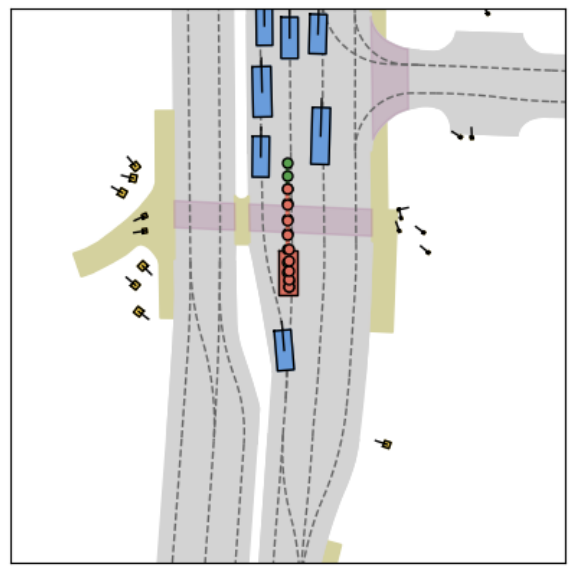}\\
    
    \vspace{4pt}
    
    \adtframe{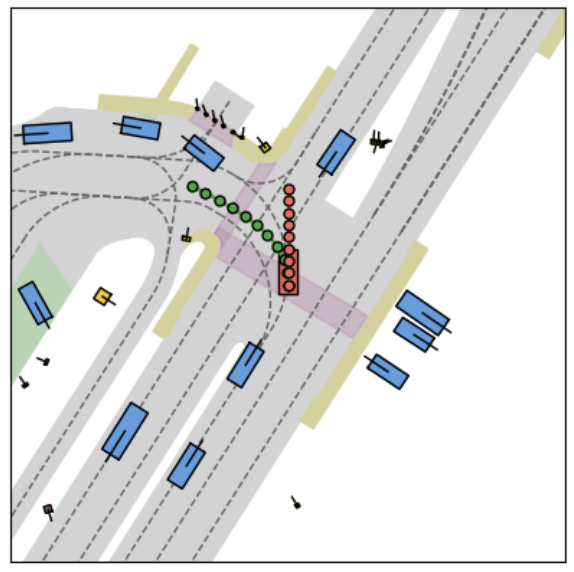}%
    \hspace{2pt}
    \adtframe{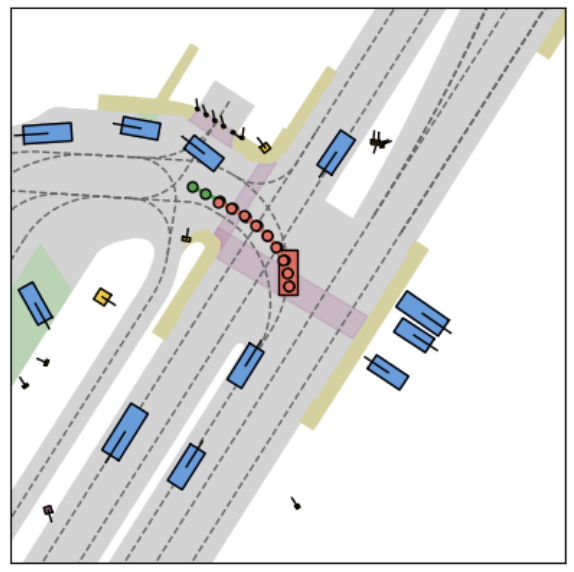}%
    \hspace{2pt}
    \adtframe{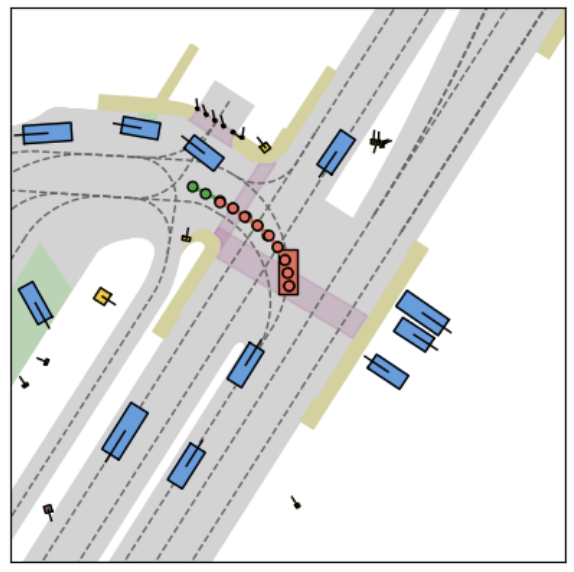}%
    \hspace{2pt}
    \adtframe{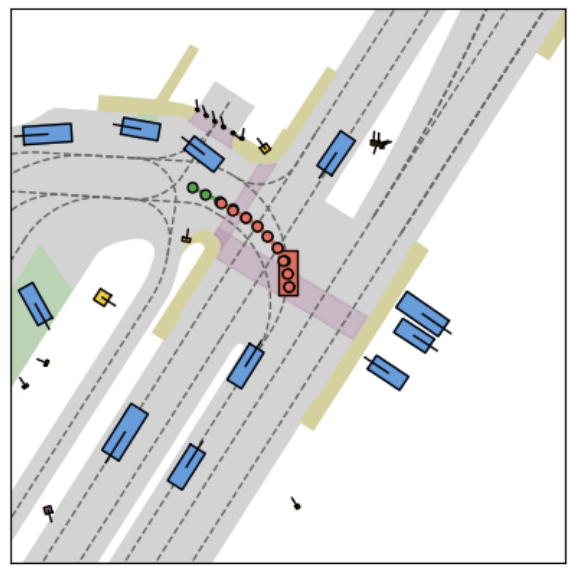}%
    \hspace{2pt}
    \adtframe{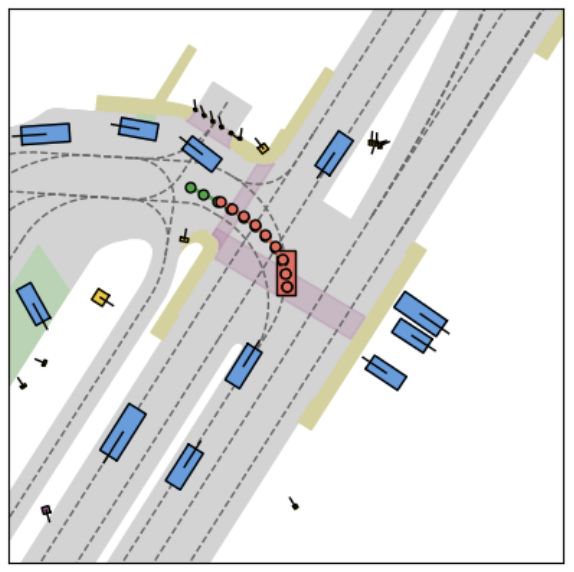}\\
    
    \vspace{4pt}
    
    \adtframe{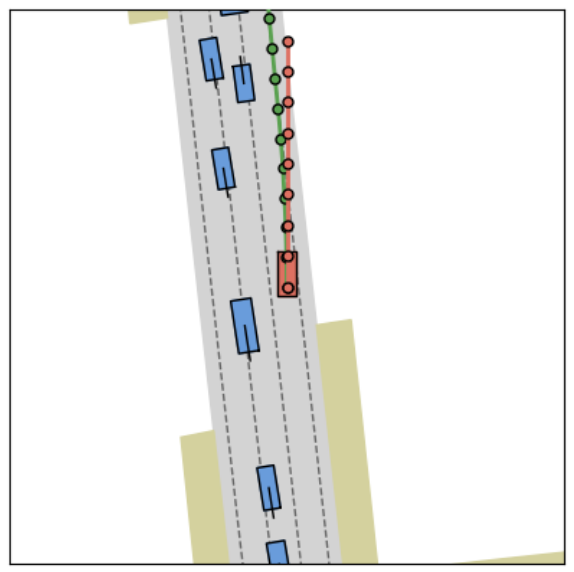}%
    \hspace{2pt}
    \adtframe{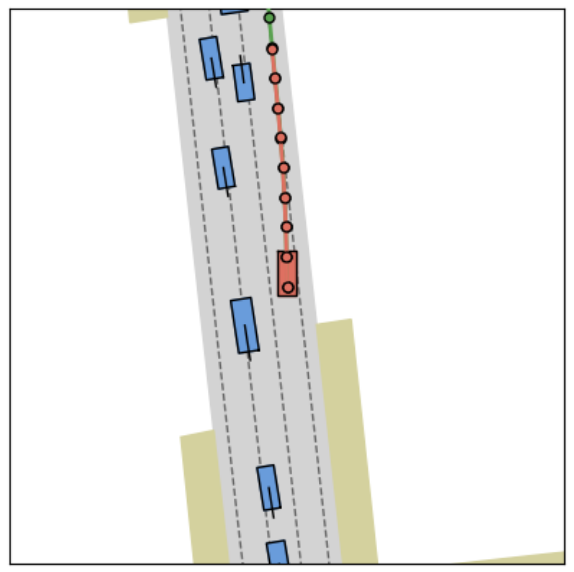}%
    \hspace{2pt}
    \adtframe{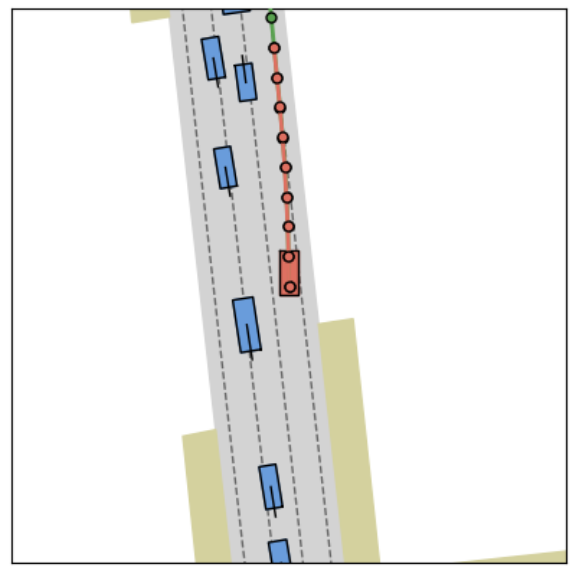}%
    \hspace{2pt}
    \adtframe{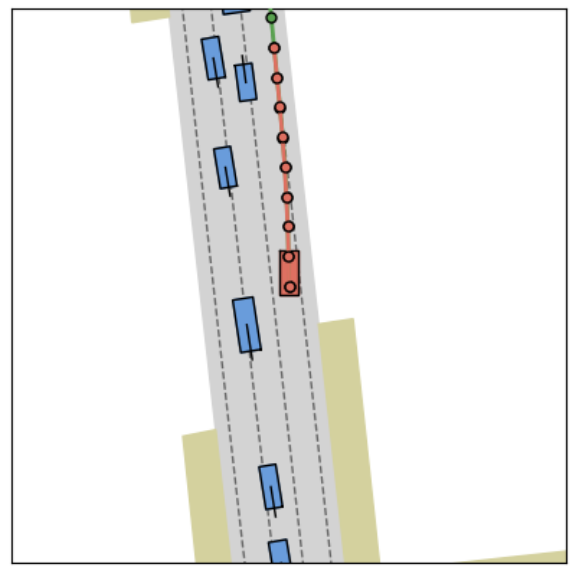}%
    \hspace{2pt}
    \adtframe{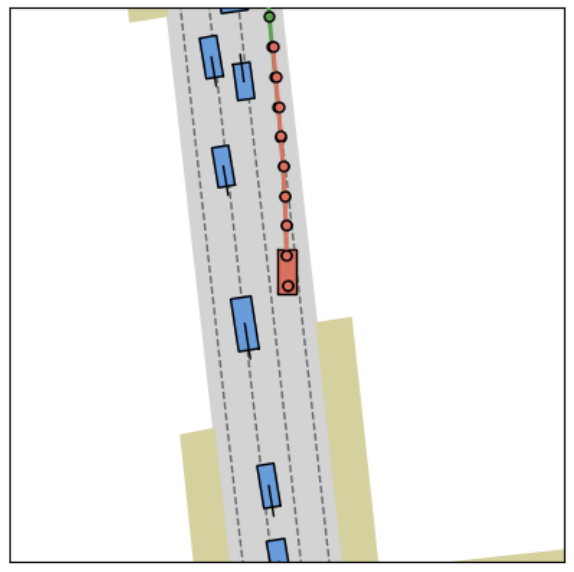}\\
    
    \vspace{4pt}
    
    \adtframe{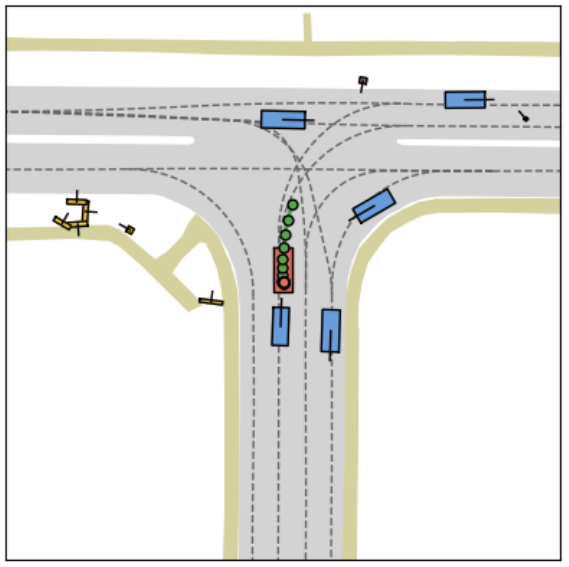}%
    \hspace{2pt}
    \adtframe{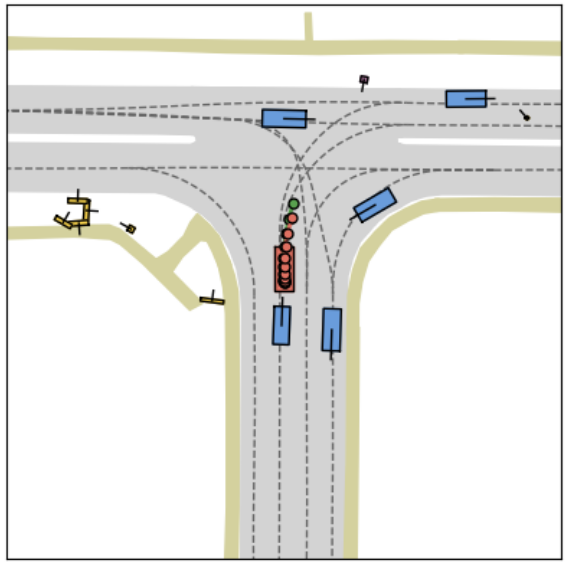}%
    \hspace{2pt}
    \adtframe{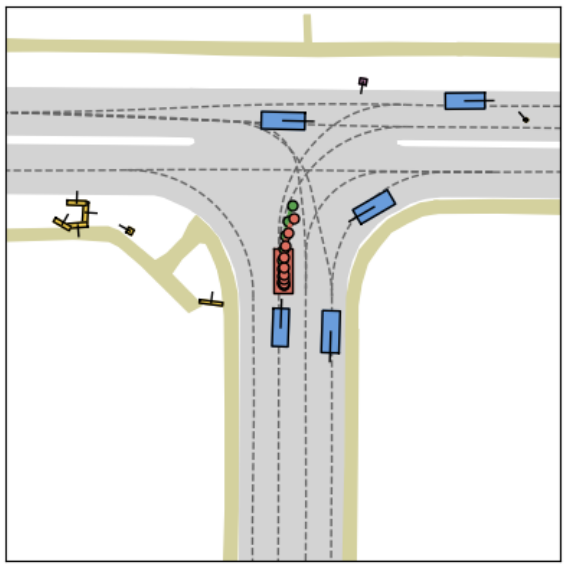}%
    \hspace{2pt}
    \adtframe{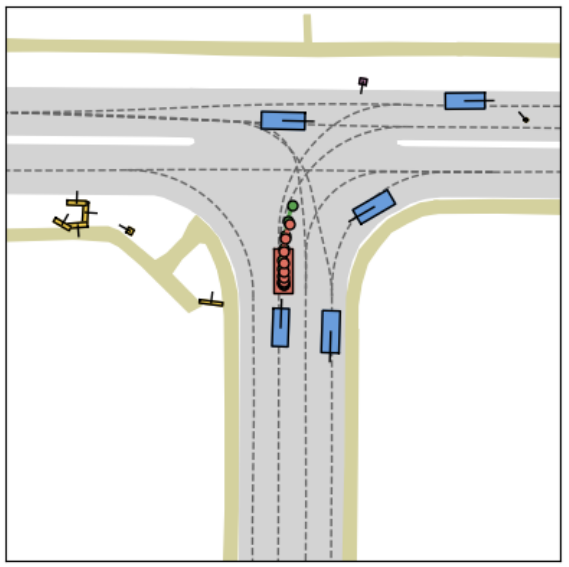}%
    \hspace{2pt}
    \adtframe{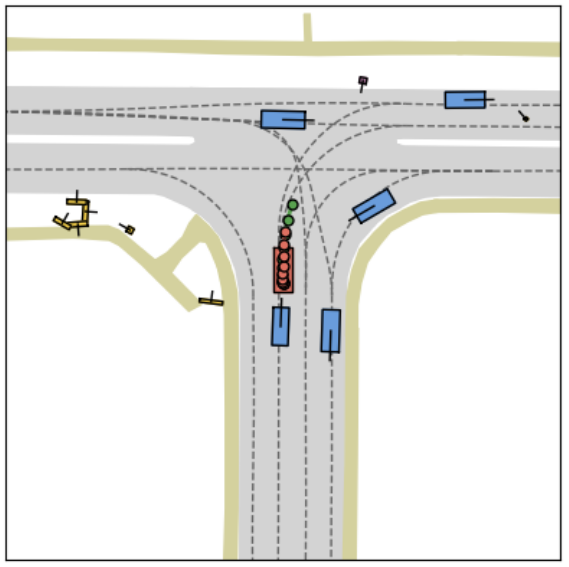}
    
    \caption{\textbf{Qualitative trajectory predictions across diverse NAVSIM driving scenes.} 
    Each row shows a different example scene. Green waypoints represent ground-truth (GT), and red waypoints represent agent predictions.}
    \label{fig:navsim_scenes_combined}
\end{figure*}

\begin{figure*}[t]
    \centering

    \includegraphics[width=0.245\textwidth]{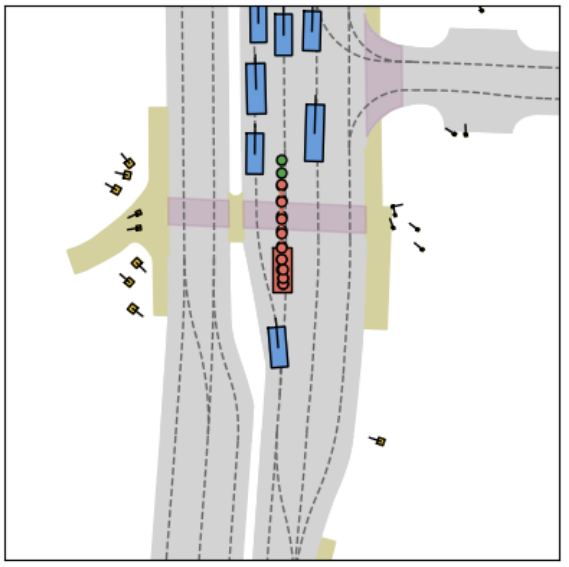}\hfill
    \includegraphics[width=0.245\textwidth]{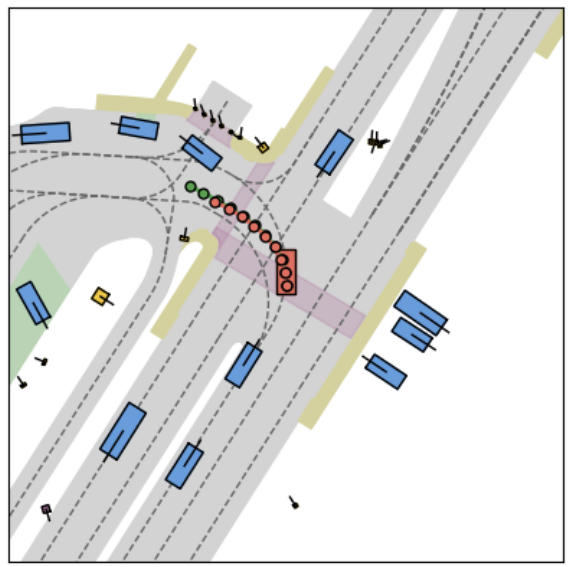}\hfill
    \includegraphics[width=0.245\textwidth]{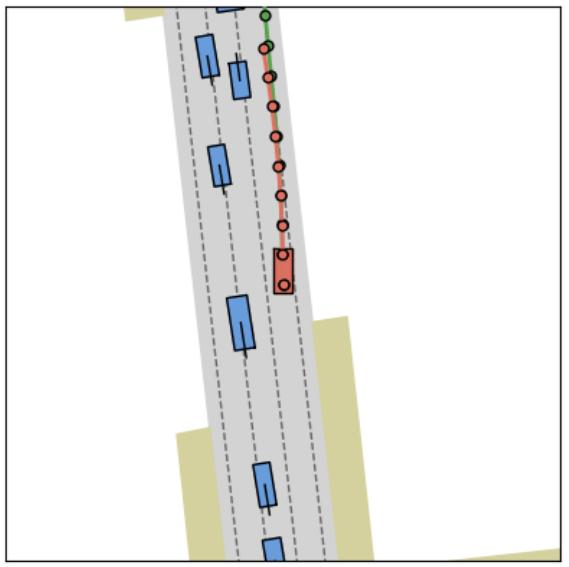}\hfill
    \includegraphics[width=0.245\textwidth]{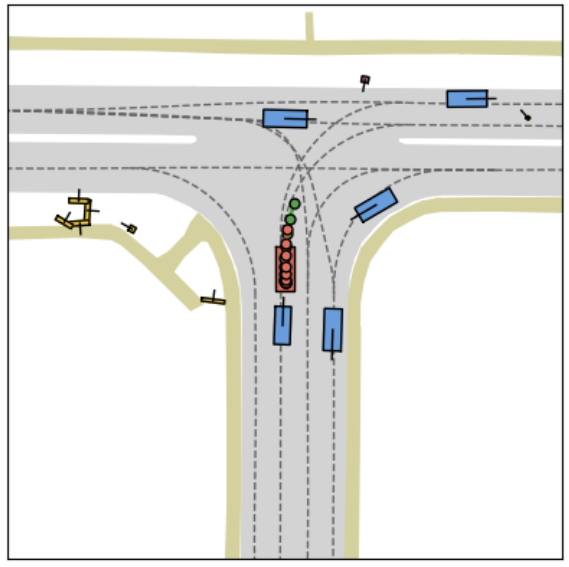}

    \caption{\textbf{LTF \cite{Chitta22TransFuser} qualitative predictions across four NAVSIM scenes.}
    From left to right: example scenes 1, 2, 3, and 4.}
    \label{fig:ltf_1row_4scenes}
\end{figure*}

\subsection{NAVSIM \texttt{navtest}: our framework w/ Euler}
\label{sec:navsim_navtest_additional}

\paragraph{CIL++ w/$\mlplift$ \& MILE w/$\mlplift$ results w/RK4.} In NAVSIM \texttt{navtest}, other agents do not react to the ego, so small improvements in the predicted trajectory with respect to the PDMS components translate directly into higher PDMS. An unconstrained MLP has sufficient flexibility to learn this dataset- and metric-induced action-to-waypoint mapping for the \texttt{navtest} protocol. In contrast, KBM/CCPP deliberately restrict the trajectory family through a vehicle-model prior; this structure is expected to matter more in reactive evaluation (e.g., \texttt{navhard}) because the ego trajectory influences the future scene, such that small differences in curvature, acceleration, and feasibility can change how other agents respond and whether the ego remains safe and compliant over time. In \texttt{navtest}, however, the scene evolution is fixed and does not react to the ego, so these interaction-related benefits are not strongly rewarded by the open-loop metrics. 

\paragraph{Results with Euler integration.}
Tab.~\ref{tab:navsim_kbm_euler} reports open-loop performance on \texttt{navtest}. Among our action-based variants, CIL++ w/ CCPP attains the best PDMS ($83.2$), improving over CIL++ w/ KBM ($81.9$). For MILE, CCPP likewise outperforms KBM ($82.3$ vs.\ $81.0$). At the submetric level, all framework-coupled variants preserve perfect comfort ($100$) and maintain strong safety/compliance (NC/DAC/TTC); the remaining PDMS differences are largely explained by small shifts in TTC and EP. Consistent with our discussion of \texttt{navtest} being non-reactive, the MLP action-to-waypoint controllers slightly outperform the structured rollouts (CIL++: $84.5$; MILE: $83.5$), as the open-loop protocol primarily rewards small improvements in the predicted trajectory under the PDMS components.

\paragraph{Euler vs.\ RK4 on \textit{navtest}.}
We next compare Euler (Tab.~\ref{tab:navsim_kbm_euler}) to RK4 (Tab.~\ref{tab:navsim_kbm_rk4}) while keeping the same action policies and evaluation protocol. Baselines and MLP mappers remain unchanged, as they do not involve numerical integration. In contrast, the vehicle-model instantiations generally benefit from RK4: KBM improves for both backbones (CIL++: $81.9\!\rightarrow\!83.3$; MILE: $81.0\!\rightarrow\!82.7$), with gains reflected mainly in higher DAC and EP (e.g., CIL++ KBM: DAC $93.7\!\rightarrow\!94.7$, EP $77.3\!\rightarrow\!78.3$; MILE KBM: DAC $92.1\!\rightarrow\!94.3$, EP $75.9\!\rightarrow\!77.4$). For CCPP, RK4 is competitive and slightly improves MILE ($82.3\!\rightarrow\!82.8$), while CIL++ shows a small decrease ($83.2\!\rightarrow\!82.5$), indicating a minor trade-off among the tightly clustered secondary metrics rather than a systematic failure mode.

\paragraph{Why RK4 tends to help.}
This trend is expected: in our framework, the predicted action horizon is converted into a waypoint trajectory by repeatedly integrating continuous-time vehicle dynamics over discrete steps. Euler introduces larger truncation error per step, which can accumulate over the $K$-step horizon and slightly distort the executed trajectory, affecting sensitive open-loop terms such as DAC/EP (and, to a lesser degree, TTC). RK4 provides a higher-order approximation of the same dynamics, reducing discretization error and yielding trajectories that more faithfully reflect the intended motion model over the evaluation horizon. In a non-reactive benchmark like \texttt{navtest}, where performance is determined by the single-shot predicted trajectory under a fixed scene evolution, this improved rollout fidelity translates directly into stronger PDMS for the model-based instantiations.

\subsection{NAVSIM \texttt{navhard}}
\label{sec:navsim_navhard_additional}

\paragraph{Discussion of submetrics in NAVSIM \texttt{navhard} w/ RK4.} The clearest separation between mappings emerges in the reactive stage (S2). In particular, \emph{drivable-area compliance} (DAC) drops markedly for the MLP couplings (e.g., CIL++: DAC $56.9$), while KBM and especially CCPP recover substantially higher DAC (CIL++: $62.4/69.9$; MILE: $75.3/73.6$). We attribute this to error accumulation under feedback: in a reactive rollout, small trajectory deviations can trigger compounding corrections that push the ego outside the drivable region, whereas KBM/CCPP provide a structured, dynamically plausible motion prior that stabilizes the executed trajectory and improves off-road avoidance. A second notable pattern is that MILE w/ CCPP achieves the strongest \emph{progress} and \emph{lane keeping} in S2 (EP $88.8$, LK $50.9$), consistent with continuous-curvature rollouts producing steadier steering while sustaining forward motion. Finally, comfort subscores are comparatively less diagnostic: \emph{HC} is near-saturated across methods ($\approx95$--$98$), and \emph{EC} can be high even when safety/compliance degrade (e.g., MLP attains EC $79.8$ in S2 but lower NC/DAC), indicating that the main practical gains of KBM/CCPP on \texttt{navhard} come from improved reactive compliance and stability rather than comfort alone.

\paragraph{Results with Euler integration.}
Tab.~\ref{tab:navhard_framework} reports performance on \textit{navhard} under the reactive pseudo-simulation protocol. Under Euler, the strongest overall results are obtained by MILE couplings, with MILE w/ CCPP achieving the best EPDMS ($26.0$) followed by MILE w/ KBM ($25.0$), while CIL++ couplings reach EPDMS $22.3$ for both KBM and CCPP. Importantly, coupling the action-based baselines with our MLP controller remains clearly less effective in the reactive setting (CIL++ w/ MLP: $21.0$; MILE w/ MLP: $24.2$), further supporting that once the scene responds to the ego, the vehicle-model structure in our KBM/CCPP instantiations becomes advantageous and consistently outperforms an unconstrained mapping. At the submetric level, the most salient gaps again occur in the reactive stage S2, particularly in safety/compliance and stability-related terms: compared to MLP, the vehicle-model instantiations maintain higher drivable-area compliance and comparable or improved at-fault collision rates, while sustaining strong progress (EP) and lane keeping (LK). Comfort terms (HC/EC) are comparatively less diagnostic, remaining high or mixed across methods and not directly tracking safety/compliance differences.

\paragraph{Euler vs.\ RK4 on \texttt{navhard}.}
We next compare Euler (Tab.~\ref{tab:navhard_framework}) to RK4 (Tab.~\ref{tab:navhard_framework_rk4}). Since CV, MLP and LTF do not depend on integration, their scores remain unchanged; the differences arise solely from the rollout used by KBM/CCPP. Across all framework-coupled variants, RK4 yields substantially higher EPDMS than Euler: for CIL++, KBM/CCPP improve from $22.3$ to $27.3$; for MILE, KBM improves from $25.0$ to $28.2$ and CCPP from $26.0$ to $29.5$. These gains are reflected in the secondary metrics that matter most in S2, where RK4 particularly strengthens compliance and stability (e.g., higher DAC and DDC, and improved LK for the best-performing MILE couplings), while maintaining competitive progress (EP) and near-saturated traffic-light compliance (TLC). In other words, the higher-order integration consistently shifts the reactive-stage metrics in the direction that the pseudo-simulation rewards: fewer off-road events and more stable execution over time.

\paragraph{Why RK4 helps under reactive pseudo-simulation.}
This trend is expected in \textit{navhard}: the pseudo-simulation rolls the scene forward under feedback, so small rollout errors can compound across the horizon and influence future interactions, especially for compliance- and collision-related terms. Euler integration introduces larger discretization error per step, which can accumulate into trajectory drift and unstable curvature evolution, increasing the likelihood of leaving the drivable area or violating lane-keeping constraints in S2. RK4 provides a higher-order approximation of the same continuous-time dynamics, reducing truncation error and yielding trajectories that better respect the intended motion model, which in turn improves stability and compliance under feedback. As a result, RK4 is systematically better aligned with the reactive evaluation protocol of \textit{navhard}, and its advantage over Euler is larger than in the non-reactive \textit{navtest} setting.

\subsection{Bench2Drive \emph{Dev10}: our framework w/ Euler}
\label{app:extrab2d_euler}

\paragraph{Best-checkpoint evaluation with multiple seeds w/ Euler.}
We conduct the same best-checkpoint, three-seed evaluation as in the RK4 setting of the main paper, but using Euler rollouts. For each method, we select the best \emph{Dev10} checkpoint and re-evaluate it under three random evaluation seeds; Table~\ref{tab:kbm_dev10_euler} reports mean\,$\pm$\,standard deviation. Under Euler rollouts, both instantiations again improve closed-loop performance over the action-only baselines. For MILE, CCPP yields the higher DS ($61.5$) while KBM attains slightly higher completion ($74.5$ RC), indicating a mild DS--RC trade-off at this best-checkpoint operating point. For CIL++, the gains are substantially larger: KBM improves DS to $66.2$ with high completion ($83.1$ RC), while CCPP reaches the strongest DS ($71.7$) and also the best RC ($84.9$), with near-zero completion variance.

\paragraph{Euler vs.\ RK4 stability.}
Comparing Tables~\ref{tab:kbm_dev10_euler} and~\ref{tab:kbm_dev10_rk4}, RK4 generally yields more stable best-checkpoint behavior, particularly for completion. For MILE, RK4 makes RC essentially deterministic for both KBM and CCPP compared to Euler ($\pm 4.8$ for KBM and $\pm 1.0$ for CCPP), while keeping DS variance comparable (RK4: $\pm 1.2$ vs.\ Euler: $\pm 2.4$/$\pm 2.7$). For CIL++, RK4 preserves the very low RC variance seen under Euler and offers slightly tighter DS dispersion for KBM ($\pm 0.8$ vs.\ $\pm 1.4$), while CCPP maintains similar DS spread ($\pm 1.6$ vs.\ $\pm 3.0$) alongside negligible RC variance. 

\paragraph{On the omission of MLP baselines in Bench2Drive.} We do not include MLP controllers in Bench2Drive because it would require non-trivial architectural modifications for CIL++ and MILE beyond the scope of this study. Moreover, Bench2Drive provides ground-truth actions, allowing us to train the original action-prediction baselines in their native form; in contrast, NAVSIM does not provide ground-truth actions, which necessitates MLP waypoint-based variants to construct comparable baselines.

\subsection{Correlation study}
\label{app:extracorrst}

\textbf{Single-lane} and \textbf{multi-lane} denote two \emph{datasets} collected in CARLA under different map complexities. Single-lane is collected in Town01, a small town that only enables single-lane driving (i.e., lane changes are not possible): we collect 15 hours of data at 10\,fps ($\sim$540K frames per camera view) under four training weathers (ClearNoon, ClearSunset, HardRainNoon, WetNoon), and evaluate generalization in Town02 under SoftRainSunset and WetSunset. In contrast, multi-lane is collected across multiple towns to include more complex scenarios such as multi-lane driving, entering/exiting highways, and crossroads: following the standard setting, we hold out Town05 for testing and collect 25 hours of data at 10\,fps from Town01--Town06 (5 hours per town; $\sim$900K frames per camera), using the same training and testing weathers as in the Town01/Town02 setting \cite{Xiao2023cilpp}.

\paragraph{Further discussion.} Following Codevilla \emph{et al.}~\yrcite{Codevilla18Offline}, we also report in Figure~\ref{fig:correlation} the correlation obtained with \emph{steering-only} error, since it is shown to correlate more strongly with driving performance than a combined steering–acceleration metric. Even in this case, our waypoint-level losses (via KBM/CCPP rollouts) achieve comparable or higher correlation, which supports using them directly as training losses rather than relying on steering-only error as a proxy metric. Finally, the trends are consistent across the single-lane and multi-lane datasets, but correlations are weaker on multi-lane due to the higher scenario complexity, as expected.

\begin{figure*}[p]
    \centering

    \includegraphics[width=0.495\textwidth]{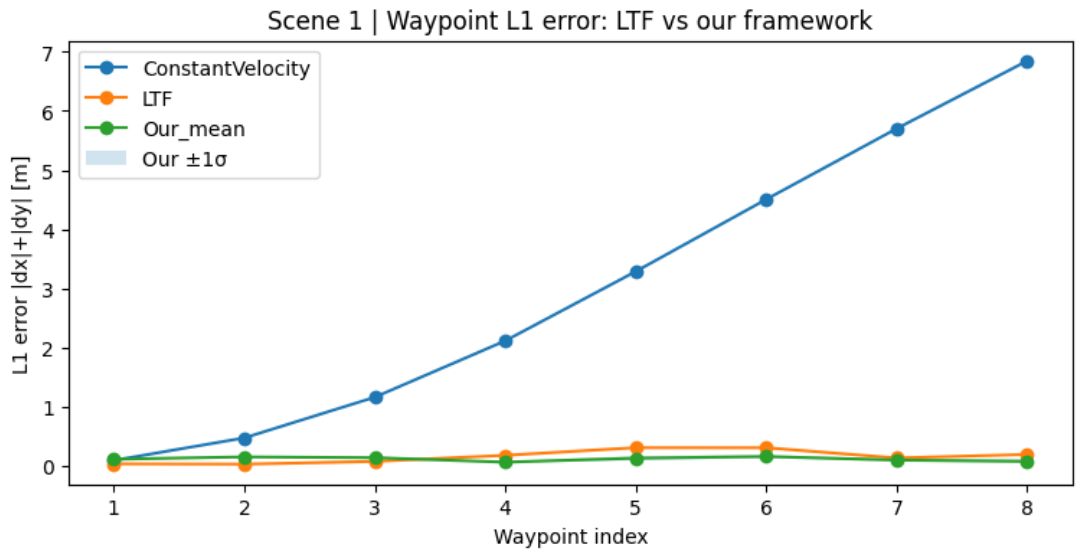}\hfill
    \includegraphics[width=0.495\textwidth]{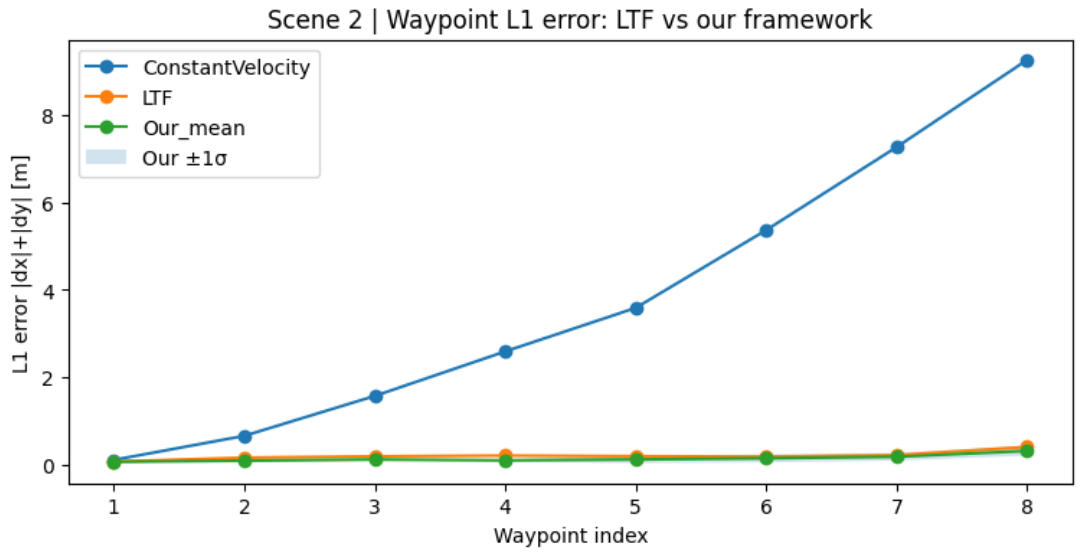}\\[4pt]

    \includegraphics[width=0.495\textwidth]{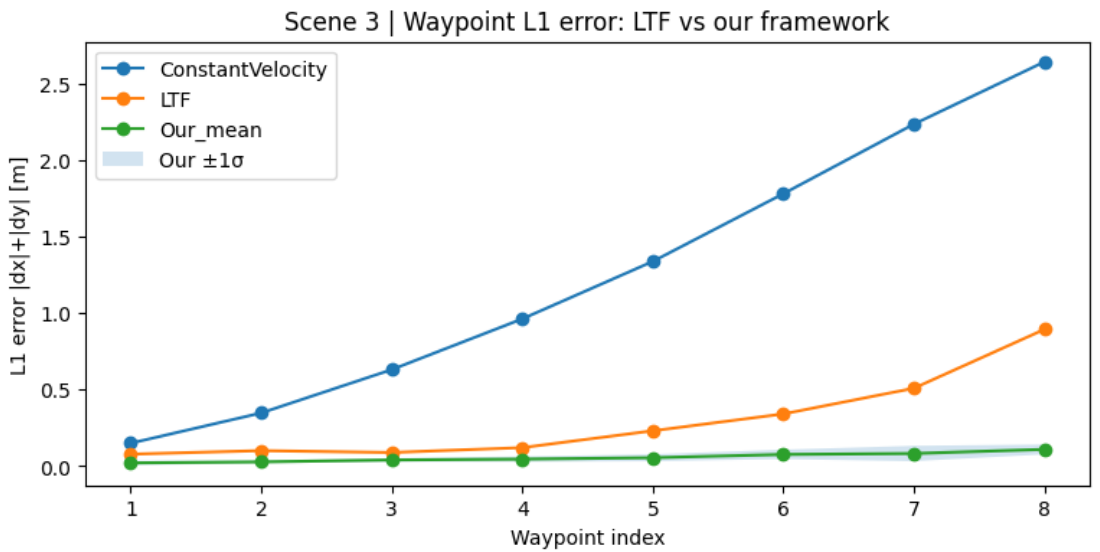}\hfill
    \includegraphics[width=0.495\textwidth]{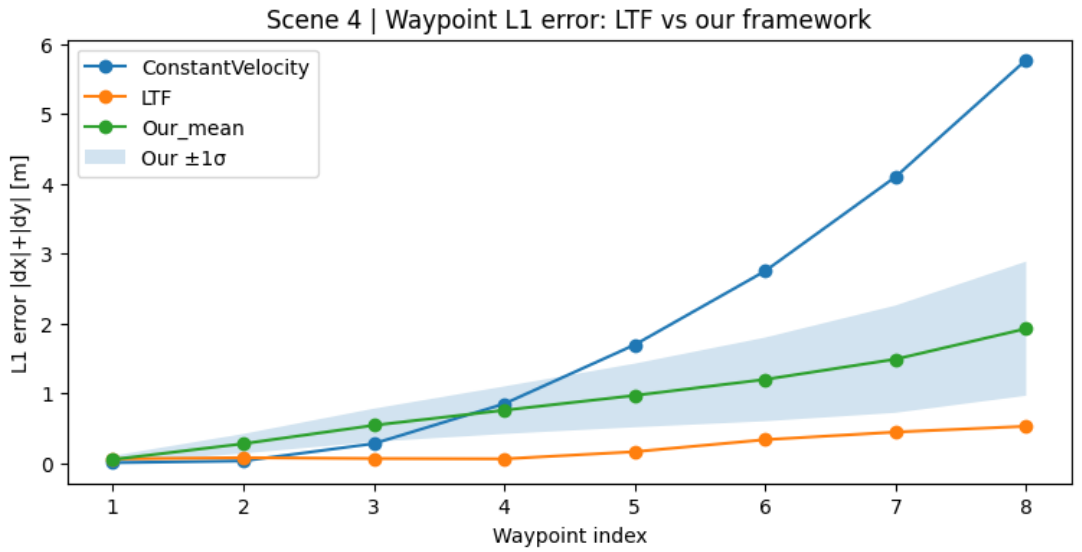}

    \caption{\textbf{Qualitative comparison across four NAVSIM scenes.}
    Each panel compares LTF vs our framework. We take the mean $\pm$ standard deviation over our four Euler-based instantiations: CIL++ w/ KBM, CIL++ w/ CCPP, MILE w/ KBM and MILE w/ CCPP.}
    \label{fig:scene_mosaic_ltf_vs_ours_big}
\end{figure*}

\begin{table*}[t]
\centering
\caption{\textbf{Framework integration on NAVSIM \textit{navtest} with Euler.} 
We take vision-only, action policies and map their action predictions into waypoints via our framework instantiations: KBM and CCPP.
The \emph{Prediction} column indicates the model's output space (\textit{waypoints} vs.\ \textit{actions}). 
We report NAVSIM \textit{navtest} six non-reactive open-loop metrics: number of collisions (NC), distance-to-all-collisions (DAC), time-to-collision (TTC), comfort (Comf.), efficiency penalty (EP), and the planning-driven metric score (PDMS) \cite{Dauner24NAVSIM}.
In addition to the CV and MLP baselines, we report results for \emph{LAW}, the current state-of-the-art vision-only method on NAVSIM \textit{navtest}.}
\label{table_2}
\label{tab:navsim_kbm_euler}
\begin{tabular}{lcccccccc}
\toprule
\textbf{Model} & \textbf{Lifting} & \textbf{Prediction} &
\textbf{NC} $\uparrow$ & \textbf{DAC} $\uparrow$ & \textbf{TTC} $\uparrow$ &
\textbf{Comf.} $\uparrow$ & \textbf{EP} $\uparrow$ & \textbf{PDMS} $\uparrow$ \\
\midrule
CV \cite{Dauner24NAVSIM}   & — & waypoints & $68.0$ & $57.8$ & $50.0$ & $100$ & $19.4$ & $20.6$ \\
Ego MLP \cite{Dauner24NAVSIM}        & — & waypoints & $93.0$ & $77.3$ & $83.6$ & $100$ & $62.8$ & $65.6$ \\
LTF \cite{Chitta22TransFuser} & — & waypoints & $97.4$ & $92.8$ & $92.4$ & $100$ & $79.0$ & $83.8$ \\
LAW \cite{Li25LatentWorldModel} & — & waypoints & $96.4$ & $95.4$ & $88.7$ & $99.9$ & $81.7$ & $84.6$ \\
\midrule
\multirow{3}{*}{CIL++ \cite{Xiao2023cilpp}} & MLP & waypoints & $96.5$ & $94.6$ & $91.8$ & $100$ & $79.3$ & $84.5$ \\
  & $\kbmlift$ & actions & $95.7$ & $93.7$ & $90.6$ & $100$ & $77.3$ & $81.9$ \\
  & $\ccpplift$ & actions & $96.0$ & $94.2$ & $90.4$ & $100$ & $78.1$ & $83.2$ \\
\midrule
\multirow{3}{*}{MILE \cite{Hu22MILE}} & MLP & waypoints & $96.7$ & $93.7$ & $91.7$ & $100$ & $78.0$ & $83.5$ \\
 & $\kbmlift$ & actions & $96.0$ & $92.1$ & $89.9$ & $100$ & $75.9$ & $81.0$ \\
 & $\ccpplift$ & actions & $95.7$ & $93.7$ & $90.3$ & $100$ & $76.9$ & $82.3$ \\
\bottomrule
\end{tabular}
\end{table*}

\begin{table*}[t]
\centering
\small
\setlength{\tabcolsep}{4pt}
\renewcommand{\arraystretch}{1.05}
\caption{ 
\textbf{Framework integration on NAVSIM \textit{navhard} w/ Euler.} We report NAVSIM \textit{navhard} metrics from the NAVSIM v2 pseudo-simulation protocol~\cite{Cao25PseudoSimulation}: the overall extended predictive driver model score (EPDMS) and its subscores, covering safety and compliance (no at-fault collisions, NC; drivable area compliance, DAC; traffic light compliance, TLC), progress (ego progress, EP), collision risk (time to collision, TTC), lane keeping (LK), and comfort (history comfort, HC; extended comfort, EC). For reference, we also include the Constant Velocity, MLP and Latent TransFuser models \cite{Chitta22TransFuser,Cao25PseudoSimulation}.}
\label{tab:navhard_framework}

\begin{tabular}{c c | c c c | c c c | c c c}
\toprule
\multirow[c]{2}{*}{\textbf{Metric}} &
\multirow[c]{2}{*}{\textbf{Stage}} &
\multirow[c]{2}{*}{\textbf{CV}} &
\multirow[c]{2}{*}{\textbf{MLP}} &
\multirow[c]{2}{*}{\textbf{LTF}} &
\multicolumn{3}{c|}{\textbf{CIL++}~\cite{Xiao2023cilpp}} &
\multicolumn{3}{c}{\textbf{MILE}~\cite{Hu22MILE}} \\
\cmidrule(lr){6-8}\cmidrule(lr){9-11}
& & & & &
\textbf{w/ MLP} & \textbf{w/ $\kbmlift$} & \textbf{w/ $\ccpplift$} &
\textbf{w/ MLP} & \textbf{w/ $\kbmlift$} & \textbf{w/ $\ccpplift$} \\
\midrule

\multirow{2}{*}{NC$\uparrow$}  & S1 & 88.8 & 93.2 & \textbf{96.2}
& 94.1 & 95.9 & 95.2
& 95.9 & 95.4 & 95.6 \\
                              & S2 & 83.2 & 77.2 & 77.7
& 76.9 & 77.9 & 76.8
& 79.3 & 77.7 & \textbf{80.2} \\
\midrule

\multirow{2}{*}{DAC$\uparrow$} & S1 & 42.8 & 55.7 & 79.5
& \textbf{80.0} & 74.7 & 77.8
& 70.9 & 73.1 & 79.3 \\
                              & S2 & 59.1 & 51.9 & 70.2
& 56.9 & 61.2 & 59.2
& \textbf{71.6} & 71.4 & 70.9 \\
\midrule

\multirow{2}{*}{DDC$\uparrow$} & S1 & 70.6 & 86.6 & \textbf{99.1}
& 97.6 & 98.1 & 98.3
& 98.1 & 97.4 & 97.6 \\
                              & S2 & 76.5 & 74.4 & 84.2
& 76.5 & 79.9 & 78.0
& \textbf{85.2} & 84.4 & 81.9 \\
\midrule

\multirow{2}{*}{TLC$\uparrow$} & S1 & 99.3 & 99.3 & 99.5
& \textbf{99.6} & \textbf{99.6} & \textbf{99.6}
& \textbf{99.6} & \textbf{99.6} & \textbf{99.6} \\
                              & S2 & 98.0 & 98.2 & 98.0
& \textbf{98.5} & 97.8 & 98.0
& 98.1 & 98.1 & 98.0 \\
\midrule

\multirow{2}{*}{EP$\uparrow$}  & S1 & 77.5 & 81.2 & 84.1
& \textbf{85.4} & 85.2 & 85.2
& 84.3 & 84.0 & 84.2 \\
                              & S2 & 71.3 & 77.1 & 85.1
& 88.5 & 87.7 & 88.2
& 87.9 & 88.1 & \textbf{89.2} \\
\midrule

\multirow{2}{*}{TTC$\uparrow$} & S1 & 87.3 & 92.2 & \textbf{95.1}
& 92.9 & 91.8 & 91.8
& 92.2 & 92.9 & 93.6 \\
                              & S2 & \textbf{81.1} & 75.0 & 75.6
& 73.7 & 75.5 & 73.3
& 75.9 & 74.4 & 75.5 \\
\midrule

\multirow{2}{*}{LK$\uparrow$}  & S1 & 78.6 & 83.5 & 94.2
& 93.6 & 90.9 & 93.8
& 92.9 & 94.0 & \textbf{94.7} \\
                              & S2 & 47.9 & 40.8 & 45.4
& 46.3 & 44.3 & 46.1
& \textbf{49.8} & 46.2 & 46.8 \\
\midrule

\multirow{2}{*}{HC$\uparrow$}  & S1 & 97.1 & 97.5 & 97.5
& \textbf{97.8} & 97.6 & 97.6
& 97.6 & 97.6 & 97.6 \\
                              & S2 & 97.1 & \textbf{97.8} & 95.7
& 96.3 & 95.6 & 96.3
& 95.9 & 95.8 & 96.0 \\
\midrule

\multirow{2}{*}{EC$\uparrow$}  & S1 & 60.4 & 77.7 & \textbf{79.1}
& 68.9 & 76.0 & 72.9
& 74.7 & 73.3 & 72.4 \\
                              & S2 & 61.9 & \textbf{79.8} & 75.9
& 60.4 & 64.2 & 62.4
& 65.5 & 63.1 & 63.0 \\
\midrule

\multicolumn{2}{c|}{\textbf{EPDMS$\uparrow$}}
& 10.9 & 12.7 & 23.1
& 21.0 & 22.3 & 22.3
& \cellcolor{blue!15}24.2 & \cellcolor{blue!30}25.0 & \cellcolor{blue!50}26.0 \\
\bottomrule
\end{tabular}%
\end{table*}

\begin{table}[t]
\centering
\footnotesize
\setlength{\tabcolsep}{4pt}
\renewcommand{\arraystretch}{1.05}

\caption{\textbf{Closed-loop performance of our Euler-based framework in Bench2Drive.} We report Bench2Drive \emph{Dev10} results with and without our Euler-based framework, averaging closed-loop metrics over three evaluation seeds. For the best checkpoint, we report mean $\pm$ standard deviation for Driving Score (\textbf{DS} $\uparrow$) and Route Completion (\textbf{RC} $\uparrow$).}
\label{tab:kbm_dev10_euler}

\begin{tabular}{lll}
\toprule
\textbf{Model} & \textbf{DS} $\uparrow$ & \textbf{RC} $\uparrow$ \\
\midrule
MILE                 & $51.7^{\pm 1.2}$ & $56.6^{\pm 1.2}$ \\
\quad w/ $\kbmlift$          & $57.9^{\pm 2.4}$ (\improve{6.2}) & $74.5^{\pm 4.8}$ (\improve{17.9})\\
\quad w/ $\ccpplift$         & $61.5^{\pm 2.7}$ (\improve{9.8}) & $72.5^{\pm 1.0}$ (\improve{15.9})\\
\midrule
CIL++               & $43.6^{\pm 1.5}$ & $52.0^{\pm 1.0}$ \\
\quad w/ $\kbmlift$          & $66.2^{\pm 1.4}$ (\improve{22.6}) & $83.1^{\pm 0.1}$ (\improve{31.1}) \\
\quad w/ $\ccpplift$         & $71.7^{\pm 3.0}$ (\improve{28.1}) & $84.9^{\pm 0.0}$ (\improve{32.9}) \\
\bottomrule
\end{tabular}
\end{table}

\section{Numerical Analysis}
\label{app:numerical_analysis}

This appendix complements Section~\ref{subsec:numerics} with a broader sweep over horizons $C_f$ and control intervals $\Delta t$, and with an explicit compute--accuracy view for CCPP substepping. Throughout, we report open-loop rollout errors on the validation set, measured as the $L_1$ position error between predicted ego-frame waypoints $\hat{w}_{t,k}\in\mathbb{R}^2$ and ground-truth waypoints $w^{\mathrm{gt}}_{t,k}\in\mathbb{R}^2$ (cf.\ Eq.~\ref{eq:training_obj}). Where available, we also report yaw error from the propagated heading state; this is \emph{not} part of the training objective, but helps interpret integration-induced drift.

\paragraph{Horizon and framerate effects.}
Figure~\ref{fig:horizon_effect} aggregates the mean $\mathrm{XY}$ $L_1$ error as prediction time increases, either by increasing the horizon length $C_f$ or by using a larger control interval $\Delta t$. At fine temporal resolutions ($\Delta t\in\{0.1,0.2\}$s), all methods remain close, with only marginal separation between Euler and RK4. At coarse temporal resolution ($\Delta t=0.5$s), errors grow rapidly with prediction time and the choice of dynamics/integrator becomes consequential: CCPP achieves consistently lower error than KBM, and RK4 improves over Euler for both models, with the gap widening for longer horizons.

\begin{figure}[ht]
    \centering
    \includegraphics[width=0.98\columnwidth]{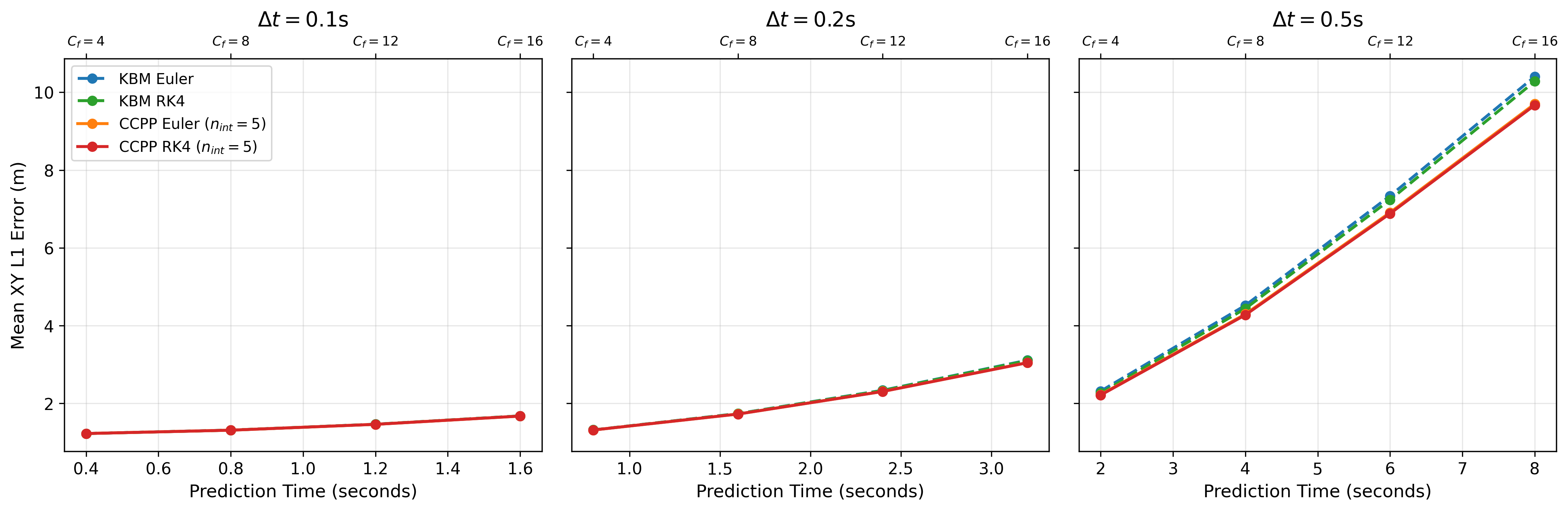}
    \caption{\textbf{Effect of horizon and framerate on rollout error.} Mean $\mathrm{XY}$ $L_1$ waypoint error versus prediction time for multiple horizons $C_f$ and control intervals $\Delta t$. At small $\Delta t$, all methods are similar; at $\Delta t=0.5$s, errors compound and CCPP (especially with RK4) maintains lower error for long-horizon prediction.}
    \label{fig:horizon_effect}
\end{figure}

\paragraph{Substepping accuracy vs.\ compute.}
CCPP refines each $\Delta t$ interval using $n_{\text{int}}$ arc-length substeps (Table~\ref{tab:design-space}). Figure~\ref{fig:n_int_flops_pareto_multi_timestep} summarises the resulting trade-off: increasing $n_{\text{int}}$ improves accuracy but increases computation. The gains are most pronounced at coarse temporal resolution, where discretisation error is largest; at higher framerates, returns diminish because all integrators already operate in a small-step regime. This motivates using moderate substepping (together with RK4) when long-horizon accuracy at low control frequency is critical.

\begin{figure}[ht]
    \centering
    \includegraphics[width=0.98\columnwidth]{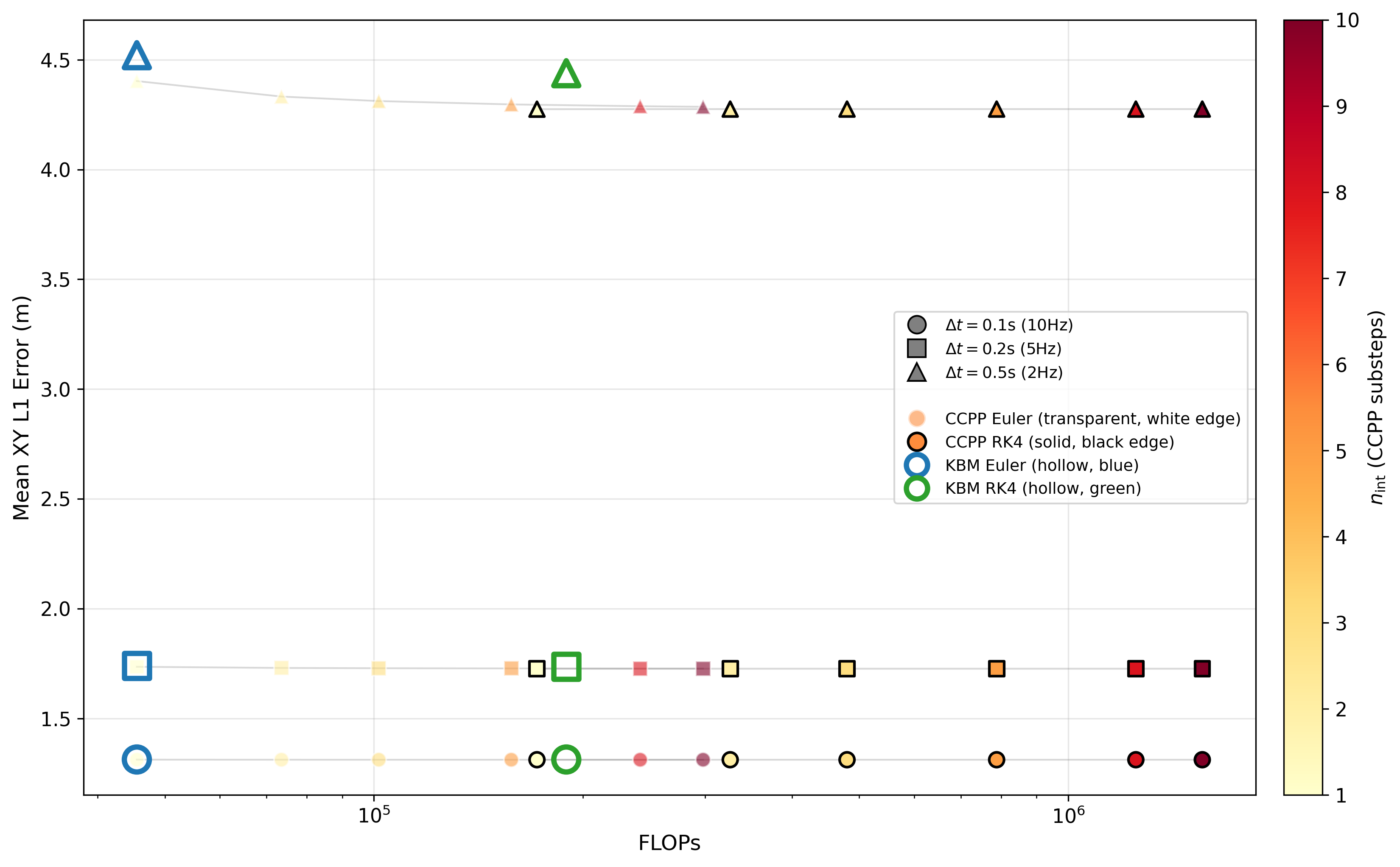}
    \caption{\textbf{CCPP substepping trade-off.} Compute--accuracy relationship when varying the number of CCPP substeps $n_{\text{int}}$ across control intervals $\Delta t$. Substepping reduces rollout error, especially at coarse $\Delta t$, at the cost of increased computation.}
    \label{fig:n_int_flops_pareto_multi_timestep}
\end{figure}

\paragraph{Per-waypoint error profiles.}
Figure~\ref{fig:per_waypoint_multi_timestep_cf8} reports mean and standard deviation of per-waypoint errors for a fixed horizon ($C_f=8$) across multiple $\Delta t$. The top row shows $\mathrm{XY}$ $L_1$ error accumulation over time, while the bottom row shows yaw error accumulation. Two trends are consistent across settings: (i) error grows with prediction time and grows faster for larger $\Delta t$, and (ii) higher-order integration (RK4) and CCPP substepping reduce both position and yaw drift, with the clearest separation appearing at $\Delta t=0.5$s. These profiles contextualise why open-loop performance differences become more impactful in reactive pseudo-simulation and closed-loop evaluation, where rollout inaccuracies compound over successive decisions.

\begin{figure}[ht]
    \centering
    \includegraphics[width=0.98\columnwidth]{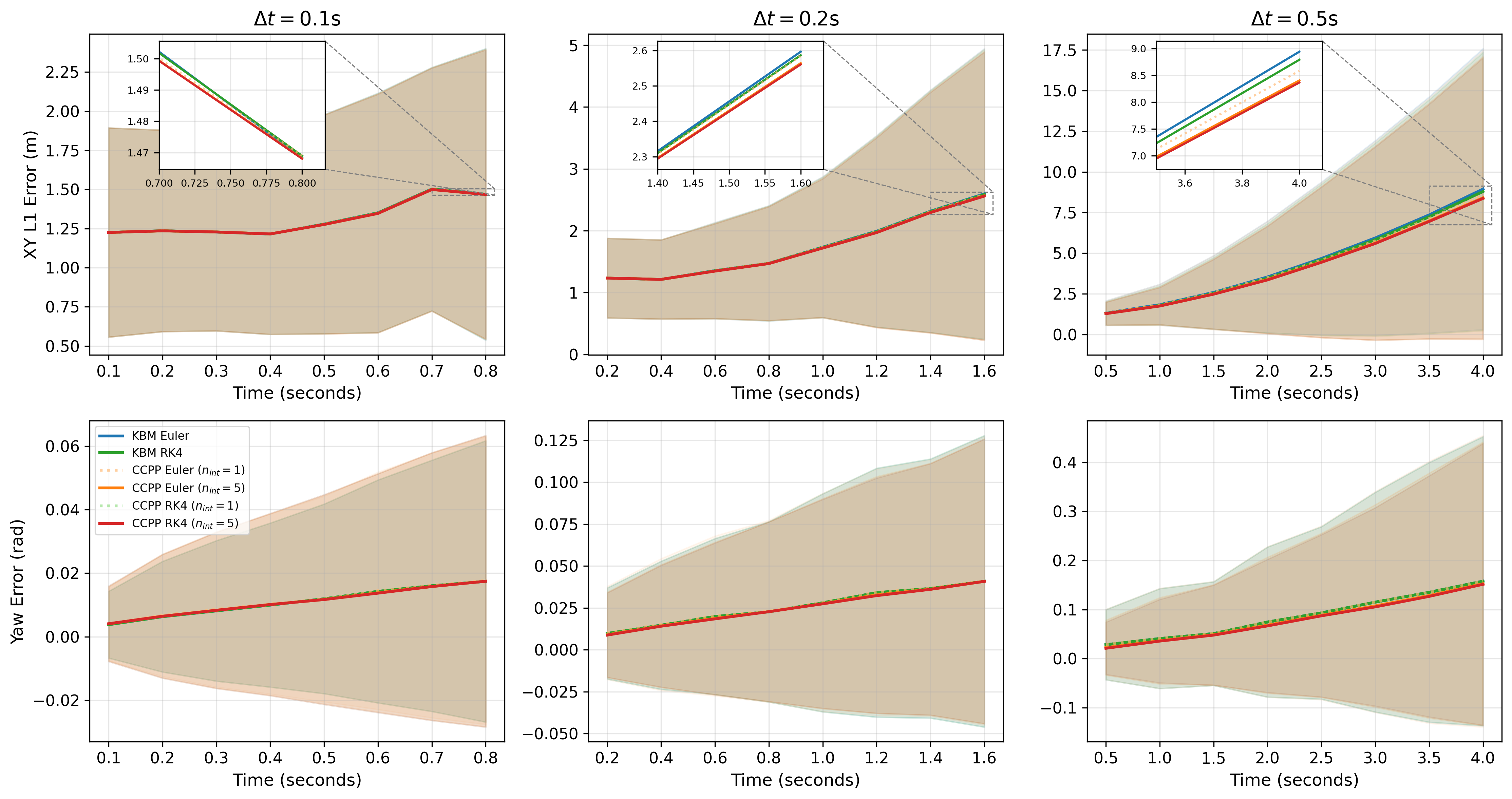}
    \caption{\textbf{Per-waypoint error accumulation for $C_f=8$ across control intervals.} Mean and standard deviation of $\mathrm{XY}$ $L_1$ error (top) and yaw error (bottom) versus time for $\Delta t\in\{0.1,0.2,0.5\}$s, comparing KBM/CCPP and Euler/RK4 (with CCPP substepping where applicable). Error growth accelerates as $\Delta t$ increases; RK4 and CCPP substepping reduce drift most clearly at coarse $\Delta t$.}
    \label{fig:per_waypoint_multi_timestep_cf8}
\end{figure}